\documentclass[sigconf]{acmart}
\usepackage{subcaption}
\usepackage[ruled,vlined]{algorithm2e}
\usepackage{multirow}
\usepackage{amsmath}
\usepackage[table]{xcolor}

\copyrightyear{2026}
\acmYear{2026}
\setcopyright{cc}
\setcctype{by}
\acmConference[RecSys '26]{20th ACM Conference on Recommender Systems}{September 27-October 02, 2026}{Minneapolis, MN, USA}
\acmBooktitle{20th ACM Conference on Recommender Systems (RecSys '26), September 27-October 02, 2026, Minneapolis, MN, USA}
\acmDOI{10.1145/3773078.3831823}
\acmISBN{979-8-4007-2284-4/2026/09}

\begin{document}

\title{DP-Rec: Towards Dynamic Patching for Efficient Long-Sequence Recommendation}

\author{Dwipam Katariya}
\authornote{Core Contributors}
\orcid{0009-0009-1058-1244}
\affiliation{%
  \institution{Capital One, AI Foundations}
  \city{McLean}
  \state{VA}
  \country{USA}
}

\author{Thomas Caputo}
\authornotemark[1]
\orcid{0009-0000-4247-9740}
\affiliation{%
  \institution{Capital One, AI Foundations}
  \city{McLean}
  \state{VA}
  \country{USA}
}

\author{Akshat Shreemali}
\authornotemark[1]
\orcid{0009-0003-4481-150X}
\affiliation{%
\institution{Capital One, AI Foundations}
  \city{New York}
  \state{NY}
  \country{USA}}

\author{Juan Manuel Origgi}
\orcid{0000-0003-3617-6019}
\affiliation{%
  \institution{Capital One, AI Foundations}
  \city{New York}
  \state{NY}
  \country{USA}
}

\author{Nikita Seleznev}
\orcid{0000-0001-7615-4493}
\affiliation{%
\institution{Capital One, AI Foundations}
  \city{Boston}
  \state{MA}
  \country{USA}}

\author{Pranab Mohanty}
\orcid{0000-0002-2224-8467}
\affiliation{%
\institution{Capital One, AI Foundations}
  \city{Seattle}
  \state{WA}
  \country{USA}}

\author{Kalanand Mishra}
\orcid{0000-0002-1832-1537}
\affiliation{%
\institution{Capital One, AI Foundations}
  \city{San Jose}
  \state{CA}
  \country{USA}}

\author{Nam Nguyen}
\orcid{0009-0008-6981-4463}
\affiliation{%
\institution{Capital One, AI Foundations}
  \city{New York}
  \state{NY}
  \country{USA}}

\author{James Montgomery}
\orcid{0009-0002-4137-8176}
\affiliation{%
\institution{Capital One, AI Foundations}
  \city{McLean}
  \state{VA}
  \country{USA}}
  
\renewcommand{\shortauthors}{Katariya et al.}

\begin{abstract}
Transformers have redefined sequential recommendation by effectively modeling dynamic user behaviors and long-range dependencies. However, they remain inherently inefficient: standard architectures operate at a fixed rate, allocating comparable computation to every item in a user’s history regardless of its information content. This leads to prohibitive computational overhead on long sequences and increased sensitivity to behavioral noise. To address this, practitioners often resort to lossy sequence compression, staged modeling, or truncation. This limits the model’s ability to leverage the full context of long histories during inference. Inspired by the recent success of Byte Latent Transformer, we propose DP-Rec, a dynamic latent patching architecture for recommendation. DP-Rec shifts from item-level modeling to patch-level modeling by segmenting interaction sequences using contrastive entropy surprise to identify informative behavioral boundaries. A lightweight patch encoder compresses these temporally contextualized segments into a reduced set of dynamic latent behavior vectors, which are then processed by a larger latent transformer and decoded for next-item prediction. Extensive experiments show that, under constrained computational budgets, DP-Rec scales effectively to long sequences and achieves a superior efficiency--accuracy trade-off over both non-compressed and fixed-size compression baselines.
\end{abstract}



\maketitle

\section{Introduction}

\begin{figure*}[t]
    \centering
    \includegraphics[width=0.95\textwidth]{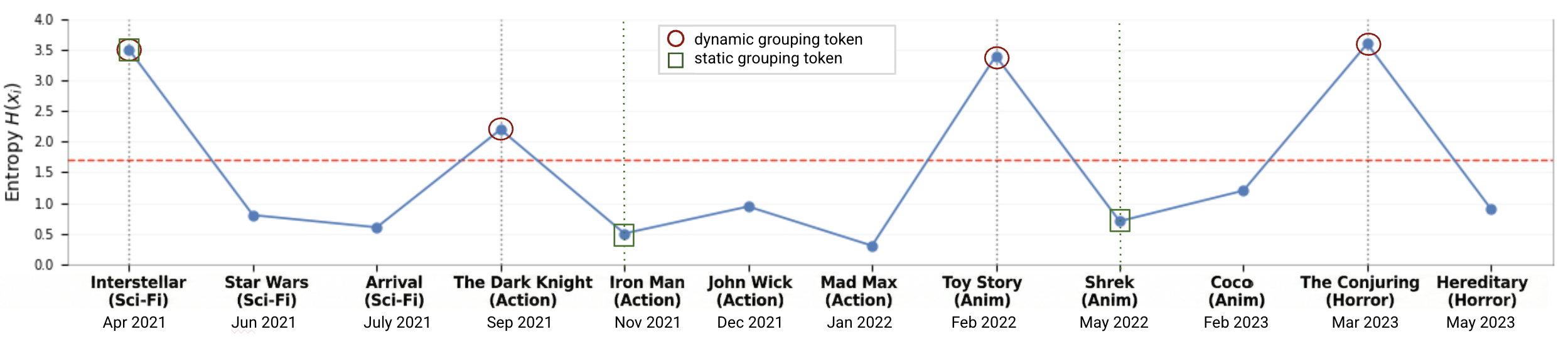}
    \Description{A line graph of entropy over a user interaction stream. A dashed line marks threshold tau. Red circles mark high-entropy peaks for Dynamic Patching; a green square marks a Fixed-size Patching interval}
    \caption{Dynamic vs Static Patching. Unlike static patching (green square), DP-Rec adaptively partitions interaction streams at contrastive entropy peaks (e.g., Interstellar, The Conjuring) (red circle). These peaks signify shifts in latent user intent. Optimized threshold $\tau$ (red dashed line) results in variable-length patching, focusing computational resources on informative behavioral boundaries rather than redundant items to maximize recommendation efficiency.} 
    \label{fig:entropies_graph}
\end{figure*}

Modern recommendation systems (RecSys) are undergoing a paradigm shift, moving from traditional feature-based architectures toward generative sequential models capable of processing long user behavior histories. These extended contexts are critical for capturing multifaceted user preferences such as long-term seasonality and recurring interests, which are often discarded by shorter lookback windows or static user embeddings~\cite{hstu, longer, survey2023, rethinkseqrec, scalingpre, personexprt}. Further, scaling sequence length in tandem with other parameters is essential to avoid performance plateaus~\cite{scaling}. Beyond raw performance gains, larger models are inherently more "sample-efficient," as they extract a higher volume of useful information from individual training examples than smaller models~\cite{scalinglaw, scalingseq, scaling}. However, the scalability of Transformer-based architectures is fundamentally constrained by the quadratic computational and memory complexity, denoted as $O(N^2 d)$, of the self-attention mechanism relative to sequence length $N$ and embedding dimension $d$~\cite{attn, personexprt}. In industrial settings, where inference must be executed for every impression within a user session, this burden is often exacerbated by the inclusion of rich item attributes (e.g., categories, brands, and textual descriptions)~\cite{hstu, timesync, transactv2}. To mitigate this, practitioners often resort to lossy sequence compression such as embeddings~\cite{fintrec, pinnerformer} or staged modeling. Representative industrial solutions include TWIN V2~\cite{twinv2}, which employs hierarchical clustering to compress life-cycle behaviors. Recent works such as LONGER~\cite{longer} have introduced token-merging schemes to compress sequence lengths. However, these methods often rely on fixed heuristics or static partitioning schemes that fail to account for the varying information density of user behavior. We argue that such compression is inherently dependent on information density; naively merging tokens based on a global length or fixed rate is sub-optimal for the varied rhythms of human behavior, requiring dynamic compression solutions as shown in Figure~\ref{fig:entropies_graph}. Dynamic compression strategies have recently demonstrated significant potential across diverse sequential domains~\cite{blt, patchtst, srsnet, reinpatch}. Specifically, the Byte Latent Transformer (BLT)~\cite{blt} uses a Perceiver-style architecture~\cite{perciever} for dynamic sequence patching. While BLT establishes the advantage of adaptive, entropy-based patching over static (fixed-vocabulary) tokenization, it derives boundaries from a full next-token softmax that is inexpensive over its 256-symbol byte vocabulary but becomes prohibitive at the million-item vocabularies of production recommender systems. Similarly, PatchTST~\cite{patchtst} utilizes patching mechanisms to capture local semantic patterns in multivariate time-series data. Despite its success in other fields, dynamic patching has yet to be investigated within the context of RecSys. We argue that a tailored latent architecture is essential to capture the idiosyncratic behavioral dynamics of recommendation data, and introduce \textbf{DP-Rec}, which adapts dynamic latent patching for sequential recommendation. Our contributions are as follows:

\begin{itemize}
    \item To our knowledge, we are the first to establish the viability of dynamic latent patching for sequential recommendation. Across \textit{KuaiRand}, \textit{ML-1M}, and \textit{ML-10M-L}, \textbf{DP-Rec} achieves Pareto-superior trade-offs, delivering either higher inference efficiency under matched accuracy, or higher accuracy within a fixed computational budget.

    \item \textbf{Dynamic Contrastive Patching}: We introduce \textbf{Contrastive Entropy Surprise}, a lightweight surrogate for Shannon entropy that resolves the computational intractability of full-vocabulary softmax in recommendations, while providing the dynamic boundary adaptability missing in static models like LONGER~\cite{longer}. This formulation demonstrates a propensity to align boundaries with shifts in interaction patterns rather than uniform positions, effectively grouping thematically coherent interaction "bursts".

    \item \textbf{Temporal-Aware Behavioral Context}: We incorporate inter-event temporal signals as first-class features within the patching mechanism. By leveraging interaction rhythms and temporal gaps alongside item-ID transitions, \textbf{DP-Rec} establishes temporal dynamics as an important determinant of behavioral context to detect informative segment boundaries.
\end{itemize}

\begin{table}[t]
\centering
\caption{Notations}
\label{tab:notation}
\small
\setlength{\tabcolsep}{4pt}
\begin{tabular}{ll}
\toprule
\textbf{Symbol} & \textbf{Meaning} \\
\midrule
$\mathcal{U}$ & Set of users \\
$\mathcal{V}$ & Set of items \\
$B$ & Mini-batch size \\
$N$ & Raw sequence length ($N_{\mathrm{max}}, N_{\mathrm{raw}}$) \\
$\mathbf{s}_u$ & Interaction sequence of user $u$, $(v_1, \dots, v_N)$ \\
$v_t$ & Item ID at position $t$ \\
$T_t$ & Absolute timestamp at position $t$ \\
$\mathbf{e}_v$ & Latent embedding vector for item $v$ ($\in \mathbb{R}^d$) \\
$\phi_{u,t}$ & Contrastive entropy surprise score at step $t$ \\
$\tau$ & Global boundary threshold \\
$\boldsymbol{\Gamma}$ & Binary boundary indicator matrix ($\in \{0,1\}^{B \times N}$) \\
$\boldsymbol{\mu}$ & Sequence padding mask matrix ($\in \{0,1\}^{B \times N}$) \\
$\mathcal{P}$ & Set of behavioral patches $\{P_1,\dots,P_M\}$ \\
$M$ & Realized patch count for a sequence \\
$\hat{M}$ & Target average patch budget \\
$P_m$ & Behavioral patch $m = [l_m, r_m]$ \\
$m(t)$ & Patch index containing event $t$ \\
$n_p$ & Average patch size \\
$\mathbf{b}_t$ & Boundary detector hidden state at position $t$ ($\in \mathbb{R}^{d_{\mathrm{bbd}}}$) \\
$\mathbf{h}_t$ & Event state from temporal local encoder ($\in \mathbb{R}^{d_{\mathrm{loc}}}$) \\
$\mathbf{z}_m$ & Input latent representation for patch $m$ ($\in \mathbb{R}^d$) \\
$\mathbf{g}_m$ & Inter-patch representation from Latent Transformer ($\in \mathbb{R}^d$) \\
$\mathbf{o}_t$ & Decoded output event state at position $t$ ($\in \mathbb{R}^d$) \\
$d$ & Latent embedding dimension \\
$d_{\mathrm{bbd}}$ & Boundary detector hidden dimension \\
$d_{\mathrm{ff}}$ & Feed-forward expansion dimension \\
$d_{\mathrm{loc}}$ & Local module hidden dimension \\
$w$ & Local attention sliding window size \\
$w_e$ & Boundary detector sliding window size \\
$L_{\mathrm{loc}}$ & Depth of local modules ($L_{\mathrm{enc}}=L_{\mathrm{dec}}=L_{\mathrm{loc}}$) \\
$L_{\mathrm{lat}}$ & Layer count for Temporal Latent Transformer \\
$f_{\mathrm{bbd}}$ & Behavioral boundary detector function \\
$f_{\mathrm{enc}}$ & Temporal local encoder function \\
$f_{\mathrm{latent}}$ & Temporal latent transformer function \\
$f_{\mathrm{dec}}$ & Local decoder function \\
\bottomrule
\end{tabular}
\end{table}

\section{DP-Rec}
\label{sec:DP-Rec}
\begin{figure*}[t]
    \centering
        \includegraphics[width=0.95\textwidth]{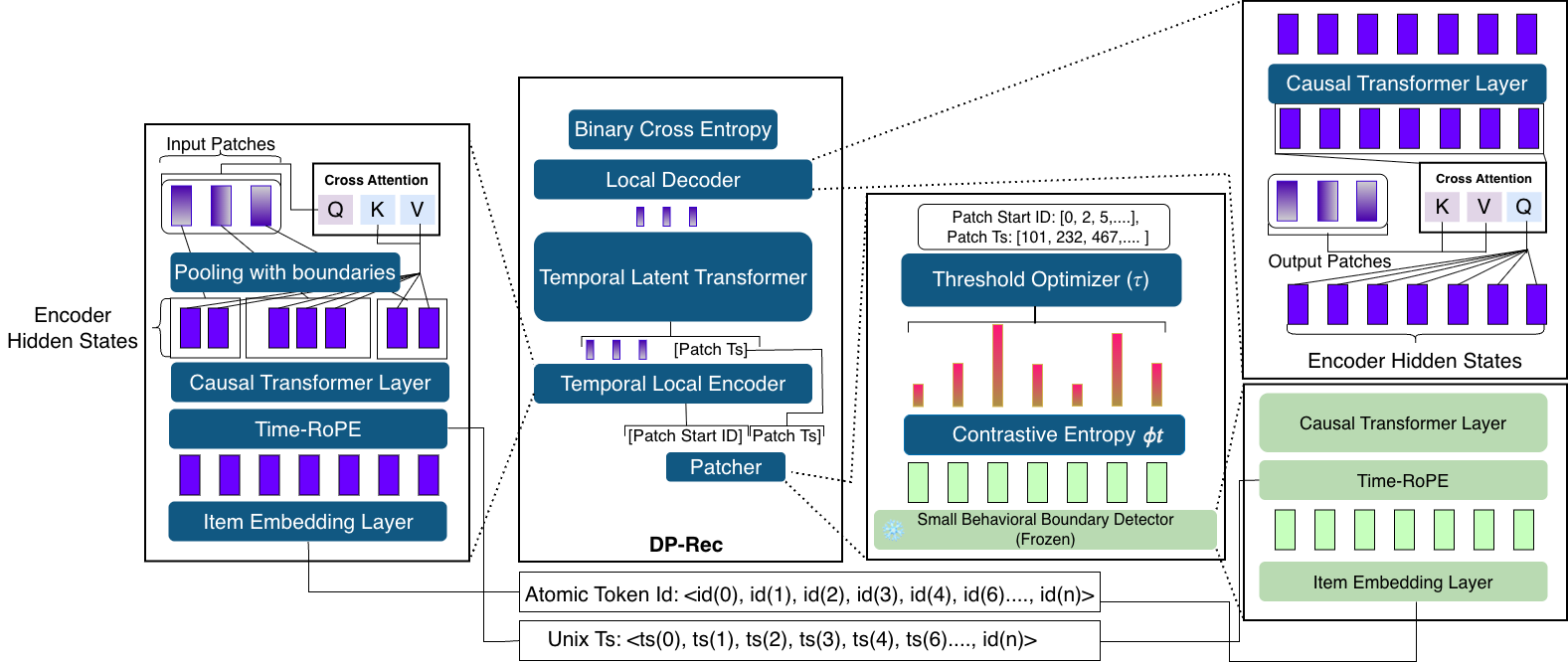}
    \caption{Overall architecture of DP-Rec. A frozen detector computes contrastive-entropy scores, enabling an adaptive threshold optimizer to partition sequences into $M$ variable-length behavioral patches. These are contextualized by a temporal local encoder and modeled via a temporal latent transformer before a local decoder projects representations back to the item space for next-item prediction.}
    \Description{Fully described in the caption and main text.}
    \label{fig:DP-Rec_architecture}
\end{figure*}
DP-Rec follows the high-level latent patching paradigm of BLT~\cite{blt} and Perceiver~\cite{perciever}, but adapts the architecture for sequential recommendation. Given a raw interaction sequence of item IDs and Unix timestamps, the behavioral boundary detector identifies patch boundaries used by the temporal local encoder to compress low-surprisal, variable-length segments into latent patch representations. The temporal latent transformer processes these patches, which a local decoder then projects back to item-level states for next-item prediction (Figure~\ref{fig:DP-Rec_architecture}). Formally, DP-Rec maps a raw interaction sequence to next-item scores through four stages:
  \begin{align*}
  \text{Boundaries:} \;\; & \mathbf{s}_u \xrightarrow{\,f_{\mathrm{bbd}},\ \text{Alg.~\ref{alg:threshold_optimization}}\,} (\boldsymbol{\Gamma}, \mathcal{P}) \\                                                     
  \text{Encode:} \;\; & \mathcal{P} \xrightarrow{\,f_{\mathrm{enc}}\,} \{\mathbf{z}_m\}_{m=1}^{M} \\
  \text{Latent:} \;\; & \{\mathbf{z}_m\} \xrightarrow{\,f_{\mathrm{latent}}\,} \{\mathbf{g}_m\}_{m=1}^{M} \\
  \text{Decode:} \;\; & \{\mathbf{g}_m\} \xrightarrow{\,f_{\mathrm{dec}}\,} \{\mathbf{o}_t\} \rightarrow \hat{y}_{u,t+1}
  \end{align*}
  where the boundary detector ($f_{\mathrm{bbd}}$) scores contrastive surprise and Algorithm~\ref{alg:threshold_optimization} turns it into boundaries $\boldsymbol{\Gamma}$ and patches $\mathcal{P}$; the temporal local encoder ($f_{\mathrm{enc}}$) compresses each patch into a latent $\mathbf{z}_m$; the temporal latent Transformer ($f_{\mathrm{latent}}$) models inter-patch dependencies as $\mathbf{g}_m$; and the local decoder ($f_{\mathrm{dec}}$) projects back to event states $\mathbf{o}_t$ for scoring.

\subsection{Temporal Rotary Positional Embeddings (Time-RoPE)}
\label{sec:timerope}

To encode interaction rhythms, we adapt Rotary Positional Embeddings~\cite{rope} to absolute timestamps $T_t$. For a latent vector $\mathbf{x} \in \mathbb{R}^d$, the transformation $\mathcal{R}(\mathbf{x}, T_t)$ rotates $d/2$ dimension pairs in the query-key space:
\begin{equation}
  \mathcal{R}(\mathbf{x}, T_t) = \mathbf{x} \odot \cos(\boldsymbol{\Theta} T_t) + \tilde{\mathbf{x}} \odot \sin(\boldsymbol{\Theta} T_t),
\end{equation}
where $\tilde{\mathbf{x}}$ denotes the vector obtained by rotating each $(x_{2i-1}, x_{2i})$ pair by $90^\circ$, i.e., $\tilde{x}_{2i-1} = -x_{2i},\ \tilde{x}_{2i} = x_{2i-1}$, and $\boldsymbol{\Theta} = \{\theta_i = 10000^{-2i/d}\}_{i=1}^{d/2}$ are the inverse frequencies. This makes the attention score between two interactions at times $T_t$ and $T_{t'}$ a function of their time lag $(T_t - T_{t'})$, giving the boundary detector, local and latent Transformers a continuous, lag-aware sense of behavioral timing.

\subsection{Behavioral Boundary Detection}
\label{sec:patching}

\begin{algorithm}[t]
\caption{Adaptive Global Threshold Optimization}
\label{alg:threshold_optimization}
  \KwIn{Surprise scores $\phi \in \mathbb{R}^{B\times N}$, padding mask $\boldsymbol{\mu}\in\{0,1\}^{B\times N}$, target average patch count $\hat{M}$}
  \KwOut{Boundary indicators $\boldsymbol{\Gamma}\in\{0,1\}^{B\times N}$, global threshold $\tau$}

$\Phi \leftarrow \{\phi_{u,t} \mid \boldsymbol{\mu}_{u,t}=1\}$\;
$N_\Phi \leftarrow |\Phi|$\;
$\bar{N} \leftarrow N_\Phi/B$\;
$M_{\mathrm{eff}} \leftarrow \min(\hat{M}, \bar{N})$\;
$\kappa \leftarrow \mathrm{round}((M_{\mathrm{eff}}-1) \cdot B)$\;
$\kappa \leftarrow \max(0, \min(\kappa, N_\Phi))$\;

\If{$\kappa=0$}{
    $\tau \leftarrow +\infty$\;
    $\boldsymbol{\Gamma}_{u,t} \leftarrow 0 \quad \forall (u,t)$\;
}
\Else{
    Sort $\Phi$ in descending order to obtain $\Phi^\downarrow$\;
    $\tau \leftarrow \Phi^\downarrow_\kappa$\;
    $\boldsymbol{\Gamma}_{u,t} \leftarrow \mathbb{I}[\phi_{u,t} \ge \tau] \cdot \boldsymbol{\mu}_{u,t} \quad \forall (u,t)$\;
}

\Return{$(\boldsymbol{\Gamma},\tau)$}\;
\end{algorithm}

To capture behavioral transitions (such as shifts in user intent or interaction rhythm), we employ a lightweight boundary detector that identifies patch boundaries prior to the primary recommendation model. While Byte-Latent Transformers~\cite{blt} derive boundaries from full-vocabulary next-token entropy, this approach is computationally prohibitive in sequential recommendation due to the massive item vocabulary $|\mathcal{V}|$. Consequently, we utilize a tractable surrogate based on \textbf{contrastive next-item surprise}.

The detector encodes the interaction sequence using a lightweight causal Transformer architecture featuring sliding-window attention and Time-RoPE applied to centered timestamps $\Delta T_t = T_t - T_1$, where $T_1$ is the timestamp of the first event in the sequence. Let $\mathbf{b}_t \in \mathbb{R}^{d_{\mathrm{bbd}}}$ denote the detector's hidden state at step $t$. The detector is a self-contained module, pre-trained then frozen, with its own parameters and item embeddings at a small hidden size $d_{\mathrm{bbd}}$ (set to $8$ in our experiments); within this subsection $\mathbf{b}_t$ and $\mathbf{e}_v$ denote these detector-internal representations, distinct from the recommendation embeddings ($\in \mathbb{R}^d$) used for scoring in Eq.~\ref{eq:bpr_loss}. To estimate surprise, we define a binary contrastive objective utilizing the ground-truth next item $v_{t+1}$ and a sampled negative item $v_t^-$. The contrastive surprise score $\phi_{u,t}$ is calculated as:
\begin{equation}
\phi_{u,t} = -p^+_{u,t}\log p^+_{u,t} - p^-_{u,t}\log p^-_{u,t},
\label{eq:surprise}
\end{equation}
where
\begin{equation}
p^+_{u,t} = \frac{\exp(\mathbf{b}_t^\top \mathbf{e}_{v_{t+1}})}{\exp(\mathbf{b}_t^\top \mathbf{e}_{v_{t+1}}) + \exp(\mathbf{b}_t^\top \mathbf{e}_{v_t^-})},
\qquad
p^-_{u,t} = 1 - p^+_{u,t}.
\label{eq:bpr_probs}
\end{equation}
A high $\phi_{u,t}$ indicates that the model assigns roughly equal probability to the positive and negative item, signaling a behavioral boundary.

Boundary selection is formulated as an \textbf{Adaptive Threshold Optimization} (Algorithm~\ref{alg:threshold_optimization}), where a global threshold $\tau$ is derived from the empirical surprise distribution to yield binary boundary indicators $\boldsymbol{\Gamma}$.

\paragraph{\textbf{Contextual Calibration}} 
\begin{figure}[t]
     \centering
     \includegraphics[width=0.48\textwidth]{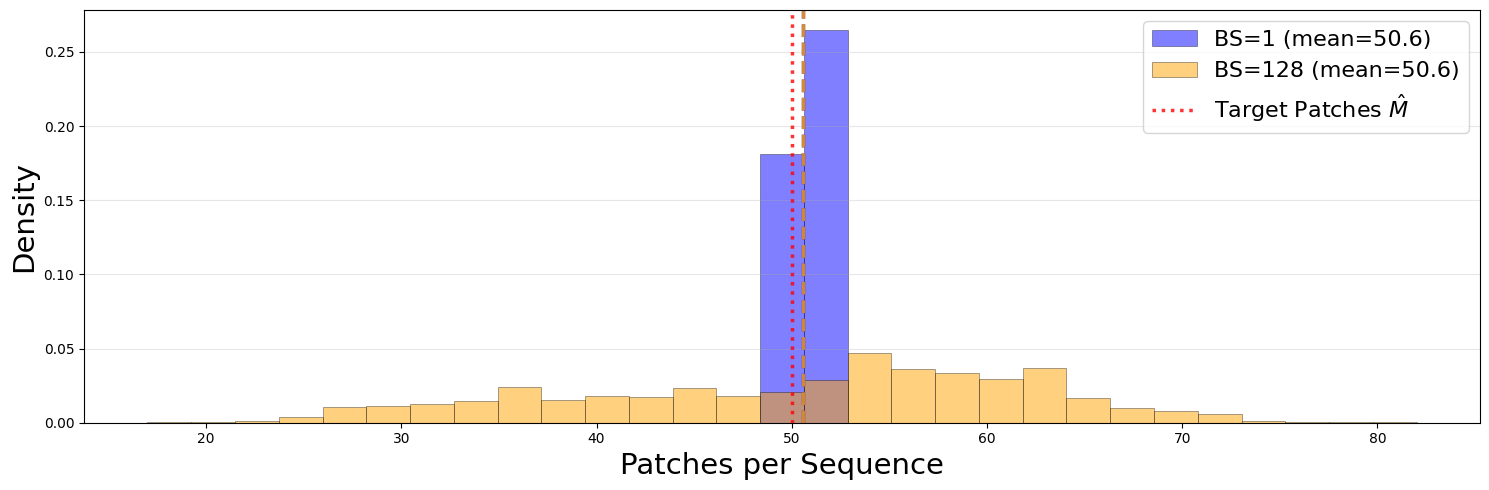}
     \caption{Impact of batch size on patch budget adherence (ML-1M, $\hat{M}=50$) for $N \ge 100$. While both configurations maintain the target mean, the Batch Size $BS=128$ distribution shows the effect of cross-sequence budget reallocation within mini-batches, contrasted with the high-precision adherence of the $BS=1$ case}
     \label{fig:ml-1m_patch_count}
     \Description{Histogram of patches per sequence. Blue bars (batch size 1) show low variance; yellow bars (batch size 128) show higher variance. A dashed vertical line marks the target patch count controlled by tau.}
\end{figure}
The optimization context varies by phase: we utilize mini-batch distributions ($B > 1$) during training and independent sequence distributions ($B=1$) during inference. We establish a combinatorial lemma (see supplementary material for details), which characterizes this behavior; notably, batch-level calibration allows high-entropy sequences to ``borrow'' capacity from lower-entropy sequences within the same mini-batch. As illustrated in Figure~\ref{fig:ml-1m_patch_count}, increasing the batch size introduces higher variance in per-sequence patch counts while maintaining the target mean. To ensure deterministic computational control and consistent evaluation, we employ per-sequence thresholding for all validation and test phases. We further note that while ties in surprise scores $\phi_{u,t}$ may lead to minor deviations from the exact cardinality $\hat{M}$, the procedure serves as a robust practical mechanism for behavioral compression across varying sequence lengths.

\subsection{Temporal Local Encoder}
\label{sec:local_encoder}

The local patch encoder transforms event-level sequences into compressed latent representations. We first compute event states $\mathbf{H} = [\mathbf{h}_1,\dots,\mathbf{h}_N]$ via a lightweight causal Transformer $f_{\mathrm{enc}}$ utilizing sliding-window attention and Time-RoPE applied to centered timestamps $\Delta T_t = T_t - T_1$. 

For each behavioral patch $P_m = [l_m, r_m]$, an initial latent patch query $\mathbf{z}_m^{(0)}$ is constructed via global max pooling over event states:
\[
\mathbf{z}_m^{(0)} = W_{\mathrm{init}} \left( \max_{t \in P_m} \mathbf{h}_t \right).
\]
These queries are iteratively refined through $L_{\mathrm{enc}}$ patch-restricted cross-attention layers, where latent representations act as queries and event states serve as keys and values:
\[
\mathbf{z}_m^{(\ell)} = \mathbf{z}_m^{(\ell-1)} + \mathrm{CrossAttn}\left( W_Q \mathbf{z}_m^{(\ell-1)}, \{W_K \mathbf{h}_t, W_V \mathbf{h}_t\}_{t \in P_m} \right).
\]
After $L_{\mathrm{enc}}$ layers, the final representation $\mathbf{z}_m = \mathbf{z}_m^{(L_{\mathrm{enc}})}$ serves as the input latent representation for patch $P_m$ to the Temporal Latent Transformer.

\subsection{Temporal Latent Transformer}
\label{sec:global_transformer}

Given the latent patch representations $\mathbf{z}_1,\dots,\mathbf{z}_M$, the temporal latent Transformer models dependencies across behavioral patches autoregressively. We use a causal patch-level Transformer, where each patch attends only to previous patches and the current patch.

Unlike the local encoder, where time is defined at the event level, the temporal latent Transformer uses a single temporal anchor per patch. For a patch $P_m=[l_m,r_m]$, we preserve its patch-end timestamp $T_{r_m}$ and use the last centered event time in that patch, $\Delta T_{r_m}$, in Time-RoPE. Thus, each latent patch is temporally anchored by the time of its final event.

The temporal latent Transformer computes:
\[
\mathbf{g}_1,\dots,\mathbf{g}_M = f_{\mathrm{latent}}\!\left(\mathbf{z}_1,\dots,\mathbf{z}_M; \, \Delta T_{r_1},\dots,\Delta T_{r_M} \right),
\]
where $\mathbf{g}_m \in \mathbb{R}^d$ is the inter-patch representation, and Time-RoPE is applied in patch-level self-attention using these patch-level temporal anchors.

Because the temporal latent Transformer operates on the compressed patch sequence rather than the full interaction history, its cost scales with the number of patches $M$ rather than the raw sequence length $N$.

\subsection{Local Decoder}
\label{sec:local_decoder}
The local decoder projects patch representations back to event-level states for scoring. It reverses the encoder's cross-attention roles~\cite{blt}, using event states as queries and patch representations $\mathbf{g}_m$ as keys and values. At layer $\ell \in \{1, \dots, L_{\mathrm{dec}}\}$, event states $\mathbf{o}_t^{(\ell)}$ are updated via:
\[
\mathbf{o}_t^{(\ell)} = \mathbf{o}_t^{(\ell-1)} + \mathrm{CrossAttn}\left(W_Q \mathbf{o}_t^{(\ell-1)}, W_K \mathbf{g}_{m(t)}, W_V \mathbf{g}_{m(t)}\right)
\]
where $m(t)$ is the patch index for event $t$. This injects global context into the event stream before a final causal Transformer yields representations $\mathbf{o}_t$ for prediction.

\paragraph{Next-item scoring.}
Hidden states $\mathbf{o}_t$ are matched against the target item embedding $\mathbf{e}_{v_{t+1}}$ and a sampled negative $\mathbf{e}_{v_t^-}$ via independent sigmoid scoring. We minimize the loss:
\begin{equation}
\mathcal{L} = -\sum_{(u,t):\, v_{t+1} \neq 0}
\left[
\log \sigma\!\left(\mathbf{o}_t^\top \mathbf{e}_{v_{t+1}}\right)
+
\log \sigma\!\left(-\mathbf{o}_t^\top \mathbf{e}_{v_t^-}\right)
\right],
\label{eq:bpr_loss}
\end{equation}
where $\sigma(\cdot)$ is the sigmoid function and the sum is over all
non-padding positions.

\section{Experiments}
\subsection{Datasets}

\begin{table}[t]
  \centering
  \caption{Dataset Summary}
  \label{tab:datasets}
  \small
  \begin{tabular}{lcccc}
    \toprule
    \textbf{Dataset} & $\mathcal{U}$ & $\mathcal{V}$ & \textbf{\# Actions} & \textbf{Avg. Actions / User} \\
    \midrule
    KuaiRand~\cite{kuairand} & 27,285 & 13,222,181 & 122,052,626 & 4,473 \\
    ML-1M~\cite{ml1m}       & 6,040  & 3,416      & 999,611     & 165    \\
    ML-10M-L~\cite{ml1m}   & 3,683  & 10,632     & 3,114,710   & 844    \\
    \bottomrule
  \end{tabular}
\end{table}

We evaluate \textbf{DP-Rec} on three datasets: \textit{ML-1M}~\cite{ml1m}, \textit{ML-10M-L}~\cite{ml1m}, and \textit{KuaiRand}~\cite{kuairand}. Due to the scarcity of public long-horizon benchmarks, we curated \textit{ML-10M-Long} (\textbf{ML-10M-L}), a long-sequence variant retaining only users with $\ge 500$ interactions from \textit{ML-10M}. Following standard protocols~\cite{sasrec, recmamba}, we apply 5-core filtering and chronological sorting. For \textit{KuaiRand}, we retain only clicked interactions following standard implicit feedback practice, yielding 4,473 average actions per user. We adopt a leave-one-out split, reserving the final and penultimate interactions for testing and validation, respectively. Preprocessed statistics are detailed in Table~\ref{tab:datasets}.

\subsection{Baselines}
\begin{itemize}
    \item \textbf{GRU4Rec}~\cite{gru4rec}: A standard RNN-based baseline utilizing Gated Recurrent Units for sequential modeling.
    
    \item \textbf{SASRec}~\cite{sasrec}: A representative and widely adopted self-attention architecture employing a causal (unidirectional) Transformer.
    
    \item \textbf{HSTU}~\cite{hstu}: A point-wise Transducer-based model incorporating relative temporal information. We evaluate it in a unimodal, causal configuration to maintain architectural parity.
    
    \item \textbf{LONGER}~\cite{longer}: A hierarchical baseline that compresses sequences into fixed-size blocks via an InnerTransformer, specifically designed for long-range dependency modeling.
\end{itemize}

\subsection{Evaluation Metrics}
We report widely used \texttt{NDCG@K=[10, 20]} and \texttt{HR@K=[5, 10, 20]}. Across all datasets, and for both validation and test, we adopt sampled evaluation: for each user we score the held-out target item against $1{,}000$ uniformly sampled negatives and rank it within the resulting $1001$-item candidate set. While the model is trained over all timesteps (Eq.~\ref{eq:bpr_loss}), evaluation follows the leave-one-out  protocol and scores a single held-out target per user. For each test user $u$, $\operatorname{rank}(y_u)$ denotes the rank of the held-out target $y_u$ among these $1001$ candidates under the model's scores. We compute
  \[
  \begin{aligned}
  \operatorname{NDCG}@K &= \frac{1}{|\mathcal{U}|}\sum_{u\in\mathcal{U}} \frac{\mathbb{I}[\operatorname{rank}(y_u)\le K]}{\log_2(\operatorname{rank}(y_u)+1)}, \\
  \operatorname{HR}@K &= \frac{1}{|\mathcal{U}|}\sum_{u\in\mathcal{U}} \mathbb{I}[\operatorname{rank}(y_u)\le K].
  \end{aligned}
  \]

\subsection{FLOPs Accounting}

Inference efficiency is measured by FLOPs per forward pass, following~\cite{hoffmann2022training} but excluding non-arithmetic costs (normalization, residuals, and embedding lookup). For a Transformer with sequence length $N$, context $c$, hidden dimension $d$, expansion $d_{\mathrm{ff}}$, and $L_{\mathrm{block}}$ layers, the base complexity is defined as:
\begin{equation}
\mathrm{FLOPs}(N,c) = N L_{\mathrm{block}} \left( 4d^2 d_{\mathrm{ff}} + 8d^2 + 2d(c+1) \right).
\label{eq:base_flops}
\end{equation}
The terms represent the feed-forward network, projections ($Q/K/V/O$), and attention aggregation, respectively. Per-user FLOPs are calculated using actual sequence lengths. We set $d_{\mathrm{ff}}=1$ across all models.

For \textbf{SASRec}, the quadratic causal context $c=(N+1)/2$ yields:
\begin{equation}
\mathrm{FLOPs}_{\mathrm{SAS}} = N L_{\mathrm{block}}(4d^2 d_{\mathrm{ff}} + 8d^2 + d(N+1)).
\end{equation}

\textbf{HSTU} utilizes the same backbone with a $1.1\times$ overhead factor to account for gating operations. \textbf{LONGER} with compression ratio $r$ and $G = \lceil N/r \rceil$ groups is modeled as the sum of inner-group encoding, cross-attention, and global self-attention:
\begin{equation}
\mathrm{FLOPs}_{\mathrm{LNG}} = G \cdot \mathrm{FLOPs}(r,r) + 2NGd + \mathrm{FLOPs}(G,G).
\end{equation}

\paragraph{\textbf{DP-Rec Complexity:}}
With $M$ latent patches, average patch size $n_p$, boundary window $w_e$, and local window $w$, we decompose \textbf{DP-Rec} cost into boundary detection ($f_{\mathrm{bbd}}$), temporal latent modeling ($f_{\mathrm{latent}}$), local encoding/decoding ($f_{\mathrm{loc}}$), and bridge cross-attention ($f_{\mathrm{brg}}$):
\begin{equation}
\mathrm{FLOPs}_{\mathrm{DP}} = N \left( f_{\mathrm{bbd}}(w_e) + \frac{f_{\mathrm{latent}}(M)}{n_p} + f_{\mathrm{loc}}(w) + f_{\mathrm{brg}} \right).
\label{eq:DP-Rec_flops}
\end{equation}
Each component is instantiated via Eq.~\ref{eq:base_flops} using its respective architecture. Notably, the boundary detector cost is included in the total inference budget despite being frozen during primary training.

\subsection{Implementation Details}
\label{sec:implementation}

\paragraph{\textbf{Training Configuration}}
All models were implemented in PyTorch and trained on an NVIDIA A100 GPU (40GB). Following established benchmarks~\cite{sasrec, gru4rec}, we trained ML-1M and ML-10M-L for 500 epochs, while KuaiRand was trained for 200 epochs. We use Adam optimizer with a learning rate of $10^{-3}$ and a batch size of 128. To control for confounding factors during hyperparameter optimization, we set dropout to $0.0$ and disable $L_2$ regularization. Best-performing checkpoints are selected based on the highest validation NDCG@10, evaluated every 10 epochs. We do not employ early stopping.

\paragraph{\textbf{Evaluation Protocols}}
We adopt two complementary evaluation regimes:
\begin{itemize}
    \item \textbf{Headline Comparison:} We report a single FLOPs-budgeted operating point per dataset, selected by the best \emph{validation} NDCG@10 under the corresponding FLOPs budget (see
  Table~\ref{tab:headline_results}).
    \item \textbf{Pareto Frontier:} We characterize the efficiency--quality trade-off by sweeping architectures for models whose inference cost can be varied naturally through sequence length or latent compression, namely SASRec~\cite{sasrec}, HSTU~\cite{hstu}, LONGER~\cite{longer}, and \textbf{DP-Rec}---we perform a wide architecture and compression sweep to characterize the efficiency--quality frontier (see Figure~\ref{fig:pareto}).
\end{itemize}

\paragraph{\textbf{Model Specifications}}

We adopt the training pipeline from~\cite{sasrectorch} and model components from RecTools~\cite{rectools} for GRU4Rec and HSTU. LONGER is a custom implementation due to unavailability of open-source code. For \textbf{DP-Rec}, we implement the architecture as described in Section~\ref{sec:DP-Rec}.
\begin{itemize}
    \item \textbf{Fixed-Horizon Models:} SASRec and HSTU employ 3 blocks with hidden dimensions $d \in \{32, 64\}$. The maximum input length $N_{\max}$ is selected from $\{25, 50, 100, 200\}$ for ML-1M and $\{25, 50, 100, 250, 500\}$ for KuaiRand and ML-10M-L. HSTU specifically enables relative-time and relative-position encodings to give it temporal information comparable to DP-Rec. GRU4Rec utilizes $\{1, 3\}$ stacks with $d=64$ and $N_{\max} = 50$ for ML-1M and $N_{\max} = 100$ for KuaiRand and ML-10M-L.
    \item \textbf{Compression-Based Models:}
For LONGER and \textbf{DP-Rec}, the input horizon $N_{\text{raw}}$ and patch budget $M$ are swept across dataset-specific ranges: $\{100, 200\}$ and $\{2, \dots, 100\}$ for ML-1M, and $\{250, 500\}$ and $\{2, \dots, 250\}$ for KuaiRand and ML-10M-L. LONGER utilizes a 1-block inner and 3-block outer Transformer. \textbf{DP-Rec} employs a 1-block local encoder and a 3-block latent Transformer ($d \in \{16, 32, 64\}$); to maintain parameter parity with SASRec, FFN expansion layer is omitted. We evaluate both max-pooling and cross-attention ($k=8$) for patch aggregation. The boundary detector is implemented as a 1-block Transformer with a minimal 8 hidden dimension across all datasets. Both boundary detector and temporal local encoder are implemented with $w_e = w = 32$ (see supplementary material for the full configuration details).
\end{itemize}

Unless otherwise stated, all models use ReLU activation, and all reported results correspond to the best validation-selected configuration for each model and dataset. Total FLOPs are computed post-hoc from the realized inference configurations on the test set to ensure a precise efficiency analysis.

\subsection{Results}
\begin{table*}[tp]
\centering
\caption{Test results under a fixed FLOPs budget matched to best-validation SASRec, reported as mean $\pm$ 95\% CI across 10 seeds. Bold and underline denote the best overall and baseline performers, and \textit{Rel Impr.} measures relative gain over the top baseline.}
\label{tab:headline_results}
\begin{tabular}{llcccccc}
\toprule
 & & \multicolumn{3}{c}{\textit{Fixed-Horizon Baselines}} & \multicolumn{2}{c}{\textit{Compression-Based Models}} & \\
\textbf{Dataset} & \textbf{Metric} & GRU4Rec~\cite{gru4rec} & SASRec~\cite{sasrec} & HSTU~\cite{hstu} & LONGER~\cite{longer} & \textbf{DP-Rec} & \textbf{Rel Impr.} \\
\midrule
\multicolumn{8}{l}{\textit{\textbf{ML-1M} ($N_{\text{max}}=50, N_{\text{raw}}=200$)}} \\
 & NDCG@10 & $0.1961 \pm 0.0023$ & \underline{$0.2176 \pm 0.0031$} & $0.2078 \pm 0.0051$ & $0.2022 \pm 0.0042$ & $\mathbf{0.2523 \pm 0.0017}$ & +16.0\% \\
 & NDCG@20 & $0.2308 \pm 0.0026$ & \underline{$0.2525 \pm 0.0029$} & $0.2410 \pm 0.0049$ & $0.2355 \pm 0.0035$ & $\mathbf{0.2883 \pm 0.0013}$ & +14.2\% \\
 & HR@5    & $0.2329 \pm 0.0057$ & \underline{$0.2628 \pm 0.0045$} & $0.2535 \pm 0.0068$ & $0.2445 \pm 0.0054$ & $\mathbf{0.3092 \pm 0.0024}$ & +17.7\% \\
 & HR@10   & $0.3545 \pm 0.0025$ & \underline{$0.3846 \pm 0.0055$} & $0.3696 \pm 0.0080$ & $0.3637 \pm 0.0039$ & $\mathbf{0.4430 \pm 0.0044}$ & +15.2\% \\
 & HR@20   & $0.4914 \pm 0.0050$ & \underline{$0.5227 \pm 0.0055$} & $0.5010 \pm 0.0086$ & $0.4956 \pm 0.0047$ & $\mathbf{0.5855 \pm 0.0027}$ & +12.0\% \\
\midrule
\multicolumn{8}{l}{\textit{\textbf{ML-10M-L} ($N_{\text{max}}=100, N_{\text{raw}}=500$)}} \\
 & NDCG@10 & $0.1566 \pm 0.0025$ & $0.1474 \pm 0.0022$ & $0.1389 \pm 0.0051$ & \underline{$0.1661 \pm 0.0058$} & $\mathbf{0.2011 \pm 0.0024}$ & +21.1\% \\
 & NDCG@20 & $0.1863 \pm 0.0020$ & $0.1766 \pm 0.0017$ & $0.1673 \pm 0.0060$ & \underline{$0.1977 \pm 0.0056$} & $\mathbf{0.2322 \pm 0.0032}$ & +17.5\% \\
 & HR@5    & $0.1890 \pm 0.0036$ & $0.1766 \pm 0.0027$ & $0.1676 \pm 0.0071$ & \underline{$0.1991 \pm 0.0080$} & $\mathbf{0.2393 \pm 0.0028}$ & +20.2\% \\
 & HR@10   & $0.2788 \pm 0.0043$ & $0.2681 \pm 0.0030$ & $0.2534 \pm 0.0070$ & \underline{$0.2984 \pm 0.0089$} & $\mathbf{0.3547 \pm 0.0043}$ & +18.9\% \\
 & HR@20   & $0.3967 \pm 0.0049$ & $0.3838 \pm 0.0049$ & $0.3658 \pm 0.0119$ & \underline{$0.4238 \pm 0.0090$} & $\mathbf{0.4894 \pm 0.0073}$ & +15.5\% \\
\midrule
\multicolumn{8}{l}{\textit{\textbf{KuaiRand} ($N_{\text{max}}=100, N_{\text{raw}}=500$)}} \\
 & NDCG@10 & $0.3842 \pm 0.0060$ & $0.3795 \pm 0.0072$ & $0.3825 \pm 0.0100$ & \underline{$0.3858 \pm 0.0093$} & $\mathbf{0.4284 \pm 0.0131}$ & +11.0\% \\
 & NDCG@20 & $0.4052 \pm 0.0058$ & $0.4012 \pm 0.0067$ & $0.4022 \pm 0.0096$ & \underline{$0.4059 \pm 0.0098$} & $\mathbf{0.4492 \pm 0.0122}$ & +10.7\% \\
 & HR@5    & $0.4585 \pm 0.0057$ & $0.4565 \pm 0.0126$ & $0.4619 \pm 0.0143$ & \underline{$0.4640 \pm 0.0095$} & $\mathbf{0.5145 \pm 0.0141}$ & +10.9\% \\
 & HR@10   & \underline{$0.5562 \pm 0.0052$} & $0.5515 \pm 0.0140$ & $0.5429 \pm 0.0156$ & $0.5534 \pm 0.0087$ & $\mathbf{0.6125 \pm 0.0126}$ & +10.1\% \\
 & HR@20   & \underline{$0.6387 \pm 0.0047$} & $0.6370 \pm 0.0095$ & $0.6210 \pm 0.0138$ & $0.6324 \pm 0.0119$ & $\mathbf{0.6941 \pm 0.0102}$ & +8.7\% \\
\bottomrule
\end{tabular}
\end{table*}
Table~\ref{tab:headline_results} compares all models at a matched, per-dataset inference budget: we fix the FLOPs budget to the SASRec configuration with the best validation NDCG@10 (a fixed horizon of $N_{\max}=50$ on \textit{ML-1M} and $100$ on \textit{ML-10M-L} and \textit{KuaiRand}), and report every other model at its strongest validation-selected configuration within that budget. GRU4Rec is the sole exception: because its recurrent cost is largely insensitive to horizon and already sits within the budget, we report its overall best configuration. This design places fixed-horizon truncation and latent compression on equal computational footing, and under these conditions DP-Rec is the top performer on every dataset and metric. Since it can process $4$--$5\times$ longer histories at the same cost, it recovers signal that truncation discards, improving over the strongest baseline by $+21.1\%$ on \textit{ML-10M-L} and $+11.0\%$ on \textit{KuaiRand} in NDCG@10. Its advantage over \textit{LONGER}, the fixed-block compression baseline, is largest on \textit{ML-1M} ($+24.8\%$ NDCG@10) consistent with the advantage of aligning patch boundaries to behavioral shifts over uniform, static segmentation. Taken together, these results indicate that DP-Rec turns behavioral-driven compression into consistent gains across every dataset and metric we evaluate.

\subsection{Model Efficiency}
\begin{figure*}[t]
     \centering
     \begin{subfigure}[b]{0.3\textwidth}
          \centering
          \includegraphics[width=\textwidth]{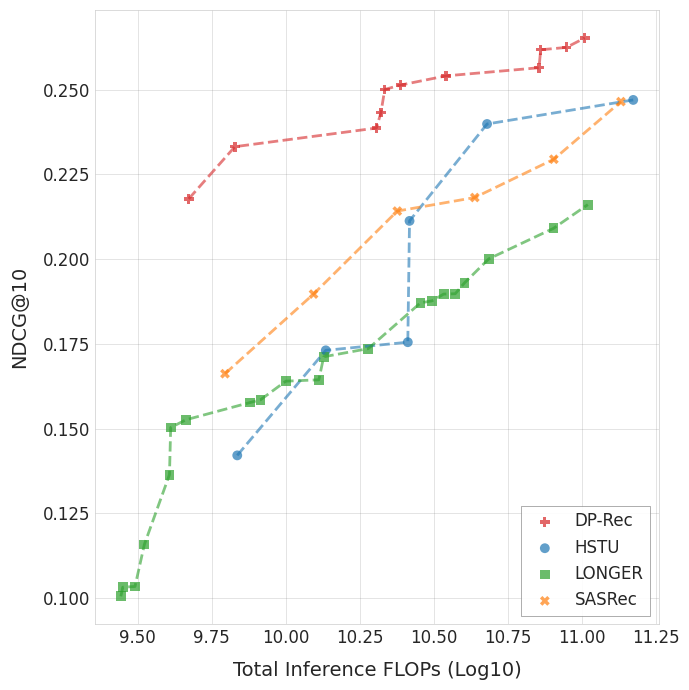}
          \caption{ML-1M}
          \label{fig:plot_left}
     \end{subfigure}
     \hfill 
     \begin{subfigure}[b]{0.3\textwidth}
          \centering
          \includegraphics[width=\textwidth]{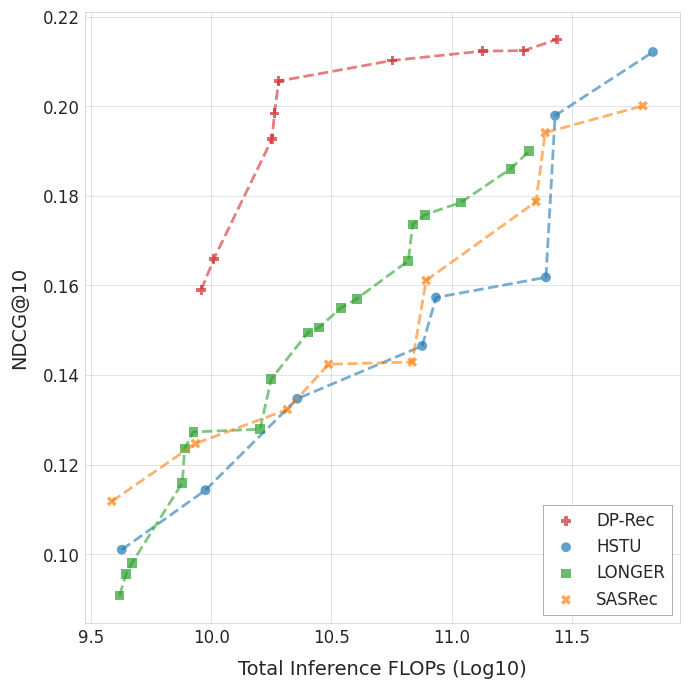}
          \caption{ML-10M-L}
          \label{fig:plot_right}
     \end{subfigure}
    \hfill
     \begin{subfigure}[b]{0.3\textwidth}
          \centering
          \includegraphics[width=\textwidth]{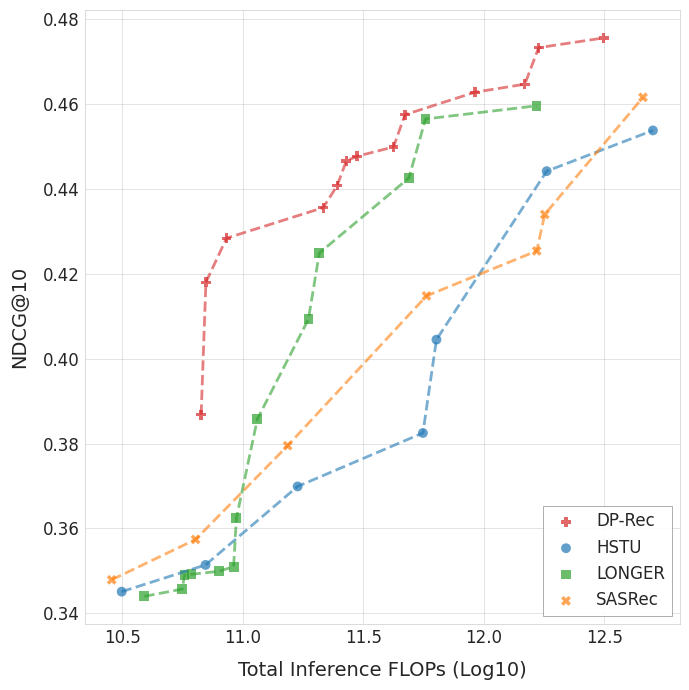}
          \caption{KuaiRand}
          \label{fig:plot_mid}
     \end{subfigure}
     \caption{Efficiency trends for budgeted inference FLOPs. We plot NDCG@10 against total inference FLOPs ($\log_{10}$ scale). \textbf{DP-Rec} (red line) consistently defines the upper-left frontier, demonstrating its ability to either (i) achieve higher recommendation quality at the same inference cost, or (ii) reach equivalent quality levels with a significantly reduced computational footprint compared to baselines.}
     \label{fig:pareto}
     \Description{Line plot showing the Pareto curve for each model. The y-axis is Normalized Discounted Cumulative Gain at 10 (NDCG@10), the x-axis is total inference Floating Point Operations on a log base 10 scale. Each model is a different color with a red line for DP-Rec above the others, indicating better accuracy when controlling for compute}
\end{figure*}
To analyze the utility--efficiency trade-off, we map test \texttt{NDCG@10} against total inference FLOPs across all model configurations. We derive the Pareto frontier for each family by sorting configurations by ascending FLOPs (breaking ties with higher utility) and retaining only those whose \texttt{NDCG@10} exceeds all lower-cost predecessors. This ensures every frontier point is non-dominated, representing the maximum achievable quality for any given computational budget. The plotted Pareto curves connect only these frontier points identifying non-dominated configurations for each model family.

\paragraph{\textbf{Pareto Dominance and Efficiency Gains}}
As illustrated in Figure~\ref{fig:pareto}, \textbf{DP-Rec} (red curve) consistently defines or expands the Pareto frontier across all benchmarks, shifting the operational envelope toward the upper-left quadrant. This dominance is characterized by two primary phenomena:
\begin{itemize}
    \item \textbf{Iso-compute Quality Gains:} At equivalent FLOP budgets, \textbf{DP-Rec} yields superior utility. For instance, on \textit{KuaiRand} (Figure~\ref{fig:pareto}c), where at $10^{11.5}$ FLOPs, \textbf{DP-Rec} achieves an NDCG@10 of $\approx 0.45$, maintaining a clear lead over the nearest baseline, \textit{LONGER}.
    \item \textbf{Iso-quality Efficiency Gains:} Conversely, \textbf{DP-Rec} satisfies target utility levels at a substantially reduced computational footprint. Due to the $\log_{10}$ scale of the x-axis, small horizontal shifts represent multifold reductions in FLOPs. For example, in \textit{ML-1M} (Figure~\ref{fig:pareto}a), reaching an NDCG@10 of $0.22$ requires $\approx 7\times$ fewer FLOPs with \textbf{DP-Rec} ($\approx 10^{9.5}$) compared to fixed-horizon models like SASRec or HSTU ($\approx 10^{10.3}$).
\end{itemize}

\paragraph{\textbf{Efficiency and Adaptive Data Distribution}}
The Pareto frontiers also reveal a dataset-dependent pattern: \textit{LONGER}'s uniform, fixed-block compression is far more competitive on \textit{KuaiRand} than on the MovieLens variants, and \textbf{DP-Rec}'s margin over it shrinks accordingly (from
  +24.8\% on ML-1M to +21.1\% on ML-10M-L to +11.0\% on KuaiRand). We attribute this to how interaction structure differs across domains: \textit{KuaiRand} is a dense, high-frequency short-video feed in which recent context is strongly predictive, so uniformly compressed representations still capture most of the relevant signal, whereas the MovieLens variants accumulate over long time spans with sparser, more diverse engagement, spreading informative interactions throughout the history where uniform compression more easily blurs them. \textbf{DP-Rec}'s adaptive boundaries preserve this dispersed signal, which is why its advantage over uniform compression is largest on the more heterogeneous MovieLens histories. 

\subsection{Ablations}
Unless otherwise noted, all ablations fix the base architecture at $L_{\mathrm{lat}}=3$, $d=64$, $L_{\mathrm{loc}}=1$, $d_{\mathrm{loc}}=32$, $w=32$, $k=8$, and $\hat{M}=2$. 
\subsubsection{\textbf{Local Module Capacity}}
\begin{table*}[t]
  \centering
  \caption{Ablation on local module capacity. Bold denotes the base configuration ($1.00\times$ compute). Parentheses denote relative change ($\Delta\%$) for NDCG and compute scale for TFLOPs.}
  \label{tab:local-capacity-ablation}
  \begin{tabular}{ccccccccc}
  \toprule
  $d_{\mathrm{loc}}$ & $L_{\mathrm{loc}}$ & \textbf{\#Params}
  & \multicolumn{2}{c}{\textbf{ML-1M}}
  & \multicolumn{2}{c}{\textbf{ML-10M-L}}
  & \multicolumn{2}{c}{\textbf{KuaiRand}} \\
  \cmidrule{4-5}\cmidrule{6-7}\cmidrule{8-9}
   &  & (non-emb) & \textbf{NDCG@10} $\uparrow$ & \textbf{TFLOPs} $\downarrow$ & \textbf{NDCG@10} $\uparrow$ & \textbf{TFLOPs} $\downarrow$ & \textbf{NDCG@10} $\uparrow$ & \textbf{TFLOPs} $\downarrow$ \\
  \midrule
  \textbf{32} & \textbf{1} & \textbf{137,920} & \textbf{0.2238} & \textbf{0.0284 ($1.00\times$)} & \textbf{0.2051} & \textbf{0.0740 ($1.00\times$)} & \textbf{0.4368} & \textbf{0.5431 ($1.00\times$)} \\
  \midrule
  16  & 1 & 99,616 & 0.1909 (-14.70\%) & 0.0100 ($0.35\times$) & 0.1927 (-6.05\%) & 0.0235 ($0.32\times$) & 0.4118 (-5.72\%) & 0.1729 ($0.32\times$) \\
  64  & 1 & 224,768 & 0.2389 (+6.75\%) & 0.0989 ($3.48\times$) & 0.2108 (+2.78\%) & 0.2663 ($3.60\times$) & 0.4405 (+0.85\%) & 1.9523 ($3.59\times$) \\
  \midrule
  32  & 2 & 159,168 & 0.2219 (-0.85\%) & 0.0543 ($1.91\times$) & 0.2095 (+2.15\%) & 0.1447 ($1.96\times$) & 0.4383 (+0.34\%) & 1.0606 ($1.95\times$) \\
  32  & 3 & 180,416 & 0.2145 (-4.16\%) & 0.0800 ($2.82\times$) & 0.2010 (-2.00\%) & 0.2153 ($2.91\times$) & 0.4407 (+0.89\%) & 1.5782 ($2.91\times$) \\
  \bottomrule
  \end{tabular}
\end{table*}
The local encoder and decoder facilitate long-horizon compression via sliding-window attention. To examine how local capacity affects quality and efficiency, we ablate hidden size ($d_{\mathrm{loc}}$) and depth ($L_{\mathrm{loc}}$) while fixing all other parameters to the base configuration above. Reducing $d_{\mathrm{loc}}$ to 16 causes substantial degradation across all datasets ($-14.70\%$, $-6.05\%$, $-5.72\%$), confirming that a minimum local capacity is necessary for effective patch compression. Expanding to $d_{\mathrm{loc}}=64$ yields meaningful gains on \textit{ML-1M} ($+6.75\%$) and \textit{ML-10M-L} ($+2.78\%$) and modest gains on \textit{KuaiRand} ($+0.85\%$), though at roughly $3.5\times$ the compute cost. Increasing depth reveals a dataset-dependent pattern: $L_{\mathrm{loc}}=2$ and $L_{\mathrm{loc}}=3$ cause mild degradation on \textit{ML-1M} ($-0.85\%$, $-4.16\%$) and \textit{ML-10M-L} ($+2.15\%$, $-2.00\%$), but provide marginal gains on \textit{KuaiRand} ($+0.34\%$, $+0.89\%$), suggesting that denser interaction sequences can leverage deeper local context. Together, $d_{\mathrm{loc}}=32$ with $L_{\mathrm{loc}}=1$ offers the best quality-efficiency trade-off across datasets, though higher capacity configurations may be worth exploring for dense, high-frequency domains.

\subsubsection{\textbf{Architectural Ablations}}
\label{sec:archablate}

\begin{table}[t]
\centering
\caption{Architectural ablations under extreme compression ($\hat{M}=2$). Parentheses denote relative change ($\Delta\%$) from the base configuration.}
\label{tab:window_ablation}
\small
\setlength{\tabcolsep}{4pt}
\begin{tabular}{lccc}
\toprule
\textbf{Variant} & \textbf{ML-1M} & \textbf{ML-10M-L} & \textbf{KuaiRand} \\
\midrule
\textbf{DP-Rec (baseline)} & \textbf{0.2238} & \textbf{0.2051} & \textbf{0.4368} \\
\midrule
w/o Time-RoPE  & 0.2018 (-9.83\%) & 0.1933 (-5.75\%) & 0.4281 (-1.99\%) \\
w/o cross-attn & 0.2424 (+8.31\%) & 0.2026 (-1.22\%) & 0.4444 (+1.74\%) \\
\midrule
$k=2$  & 0.2158 (-3.57\%) & 0.2016 (-1.71\%) & 0.4389 (+0.48\%) \\
$k=32$ & 0.2302 (+2.86\%) & 0.2012 (-1.90\%) & 0.4344 (-0.55\%) \\
\midrule
$w=8$  & 0.1956 (-12.60\%) & 0.1916 (-6.58\%) & 0.4006 (-8.29\%) \\
$w=16$ & 0.2162 (-3.40\%)  & 0.1988 (-3.07\%)  & 0.4253 (-2.63\%) \\
$w=64$ & 0.2243 (+0.22\%)  & 0.2065 (+0.68\%)  & 0.4527 (+3.64\%) \\
\bottomrule
\end{tabular}
\end{table}

We evaluate the sensitivity of \textbf{DP-Rec} to local design choices under extreme compression ($\hat{M}=2$), fixing all other parameters to the base configuration above.

\textbf{Time-RoPE} removal causes consistent degradation across all datasets (Table~\ref{tab:window_ablation}), with losses of $-9.83\%$ and $-5.75\%$ on \textit{ML-1M} and \textit{ML-10M-L}, and a smaller $-1.99\%$ on \textit{KuaiRand}, confirming that continuous temporal encoding is important for capturing interaction timing, particularly on sparser MovieLens sequences.

\textbf{Cross-attention aggregator} ($k$, the number of latent queries per patch): its effect is consistently dataset-dependent at extreme compression. On \textit{ML-1M}, removing cross-attention in favour of max-pooling yields $+8.31\%$, and on \textit{KuaiRand} it similarly yields $+1.74\%$, suggesting that at $\hat{M}=2$ with only two coarse patches, max-pool aggregation is sufficient or preferable for both short and dense sequences. Only on \textit{ML-10M-L} does cross-attention remain beneficial ($-1.22\%$ without it). Varying $k$ shows small, inconsistent effects across datasets at $\hat{M}=2$, with marginal gains for $k=32$ on \textit{ML-1M} ($+2.86\%$) and near-neutral results elsewhere, confirming $k=8$ as a reasonable default but suggesting cross-attention capacity matters more at higher patch budgets.

\textbf{Sliding-window size} ($w$) exhibits a clear threshold effect: $w=8$ causes substantial utility loss across all datasets (up to $-12.60\%$ on \textit{ML-1M} and $-8.29\%$ on \textit{KuaiRand}), as insufficient local context prevents accurate resolution of behavioral shifts. $w=64$ yields consistent gains across all three datasets ($+0.22\%$, $+0.68\%$, $+3.64\%$) but at higher computational cost; $w=32$ offers the best efficiency trade-off.

\subsubsection{\textbf{Impact of Patching Scheme}}                            \begin{table}[t]
\centering
\caption{Patching-scheme ablation (Random, Fixed, Contrastive-Entropy) at $\hat{M}=50$ for \textit{ML-1M} and $\hat{M}=100$ for \textit{ML-10M-L} and \textit{KuaiRand}, varying only the boundary selection rule. Single-seed NDCG@10 scores are reported; bold marks the best per dataset.}
\label{tab:ablation_boundaries}
\begin{tabular}{lccc}
\toprule
\textbf{Dataset} & \textbf{Random} & \textbf{Fixed} & \textbf{CE} \\
\midrule
ML-1M    & 0.2400 & 0.2437 & \textbf{0.2540} \\
ML-10M-L & 0.1946 & 0.1982 & \textbf{0.2029} \\
KuaiRand & 0.4588 & 0.4634 & \textbf{0.4761} \\
\bottomrule
\end{tabular}
\end{table}

To isolate whether DP-Rec's gains stem from dynamic boundary placement rather than sequence compression alone, we vary only the boundary selection rule: \textbf{Random} (uniform random), \textbf{Fixed} (equal-sized blocks, matching LONGER's chunking~\cite{longer}), and \textbf{Contrastive Entropy (CE)} (boundaries at high-surprisal transitions; Section~\ref{sec:patching}). We evaluate at $\hat{M}=50$ for \textit{ML-1M} and $\hat{M}=100$ for \textit{ML-10M-L} and \textit{KuaiRand}, providing sufficient patch granularity for the three boundary rules to express meaningful behavioral differences. CE consistently outperforms both baselines across all three datasets (Table~\ref{tab:ablation_boundaries}). Fixed improves over Random on all three datasets ($+1.5\%$/$+1.8\%$/$+1.0\%$ on ML-1M/ML-10M-L/KuaiRand) by preserving temporal locality. CE adds a further $+4.2\%$/$+2.4\%$/$+2.7\%$ over Fixed by aligning boundaries with genuine intent shifts, confirming that high-surprisal boundary detection captures meaningful behavioral transitions across all three domains. Note that these configurations are not constrained to the headline FLOPs budget and are thus not directly comparable to Table~\ref{tab:headline_results}.

\subsection{Sensitivity to Target Patch Budget}
\label{sec:max_patches}

\begin{table}[t]
\centering
\caption{Mean test NDCG@10 vs.\ target patch budget $\hat{M}$, at fixed $N_{\text{raw}}=50$ for \textit{ML-1M} and $N_{\text{raw}}=100$ for \textit{ML-10M-L} and \textit{KuaiRand}. Cross-attention is disabled to isolate patch budget sensitivity from cross-attention capacity effects. Bold marks best per dataset.}
\label{tab:target_patches}
\small
\begin{tabular}{cccc}
\toprule
\textbf{Target Patches ($\hat{M}$)} & \textbf{ML-1M} & \textbf{ML-10M-L} & \textbf{KuaiRand} \\
\midrule
2   & 0.2346 & 0.2050 & 0.4405 \\
5   & 0.2446 & 0.2026 & 0.4414 \\
10  & 0.2414 & 0.2070 & 0.4504 \\
50  & \textbf{0.2499} & 0.1996 & \textbf{0.4578} \\
100 & 0.2486 & 0.1949 & 0.4502 \\
250 &    --     &\textbf{0.2088} & 0.4417 \\
\bottomrule
\end{tabular}
\end{table}

The target patch budget $\hat{M}$ controls the bottleneck between local interaction processing and the temporal latent transformer. Table~\ref{tab:target_patches} reports test NDCG@10 at each budget, at fixed $N_{\text{raw}}=50$ for \textit{ML-1M} and $N_{\text{raw}}=100$ for \textit{ML-10M-L} and \textit{KuaiRand}. Cross-attention aggregation is disabled for this ablation to isolate patch budget sensitivity from the confounding effect of cross-attention capacity becoming degenerate when $k > n_p$ at high $\hat{M}$ (Section~\ref{sec:archablate}). Both \textit{ML-1M} and \textit{KuaiRand} peak at $\hat{M}=50$, while \textit{ML-10M-L} continues to improve up to $\hat{M}=250$, consistent with its longer interaction histories requiring finer-grained segmentation. Despite this variation, DP-Rec is highly resilient to aggressive compression: even at $\hat{M}=2$ it retains ${\sim}94\%$ of peak NDCG@10 on \textit{ML-1M}, ${\sim}98\%$ on \textit{ML-10M-L}, and ${\sim}96\%$ on \textit{KuaiRand}, confirming that the contrastive boundary detector packs the most salient behavioral shifts into very few patches.

\section{Qualitative Analysis}
\label{sec:qualitative_grouping}
\begin{figure}[t]
    \centering
    \includegraphics[width=\linewidth]{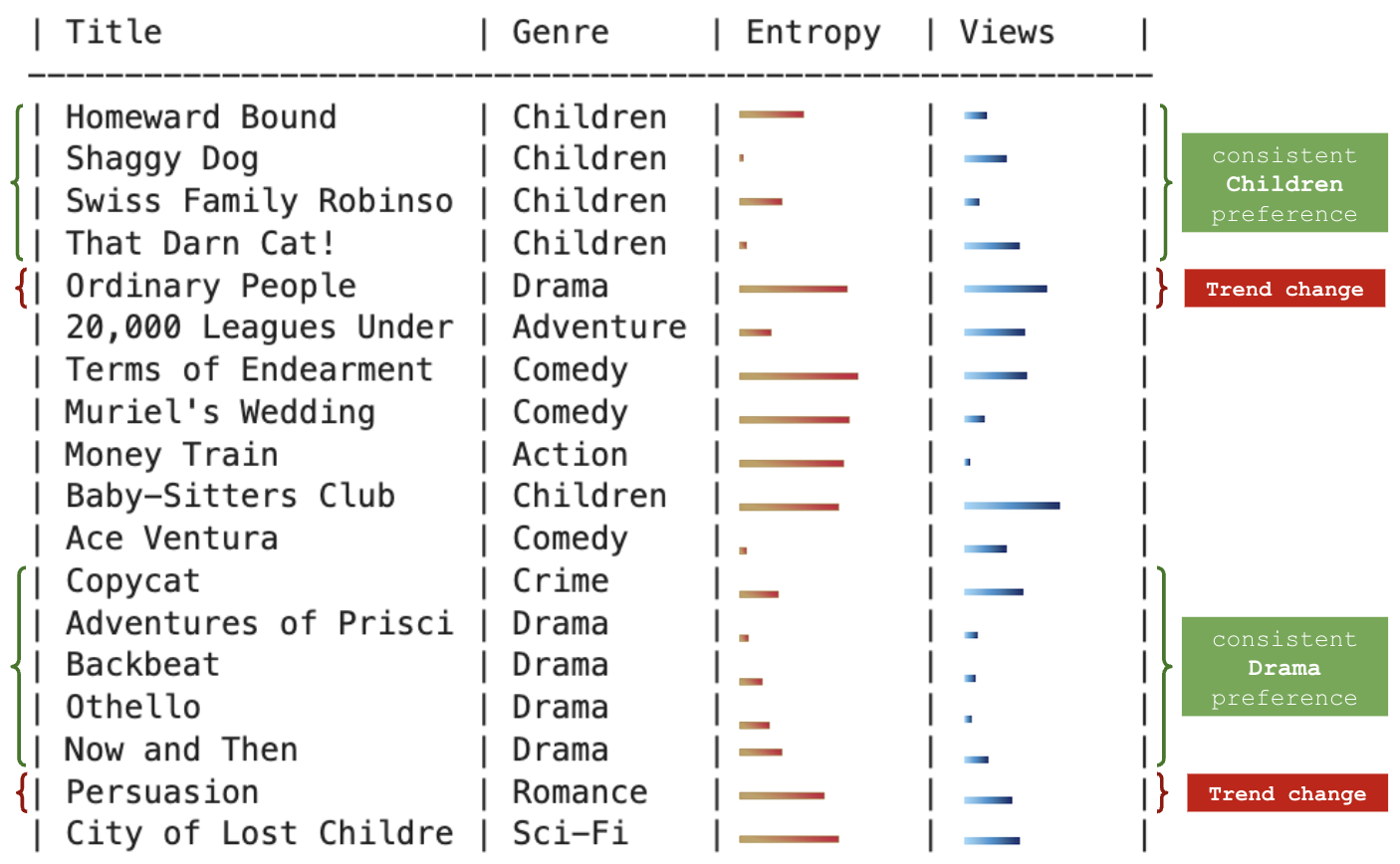}
    \caption{Qualitative analysis of contrastive entropy surprise (red) and item views (blue) for \textit{ML-1M}. Entropy remains low during stable genre engagement but spikes sharply at out-of-distribution transitions, demonstrating how \textbf{DP-Rec} adaptively allocates boundaries to salient behavioral shifts.}
    \label{fig:analysis}
    \Description{Table of viewed movies with title, genre, entropy, and views columns shown as inline bars. Green brackets group genre-stable rows with low entropy; red brackets mark genre shifts with high entropy.}
\end{figure}
Visualizations of interaction sequences in Figure~\ref{fig:analysis} confirm that \textbf{DP-Rec} functions as an adaptive segmentation module, partitioning user histories into variable-length behavioral patches. We observe that contrastive entropy surprise serves as a robust proxy for these latent transitions: scores remain suppressed during thematically stable episodes, such as consistent genre engagement in \textit{ML-1M}, allowing the local encoder to condense redundant interactions. Conversely, sharp entropy spikes occur during ``out-of-distribution'' events, including abrupt genre shifts or changes in interaction rhythm. This dynamic alignment ensures that latent representations are allocated according to behavioral complexity rather than fixed temporal intervals.

\section{Limitations and Future Work}

While \textbf{DP-Rec} demonstrates strong efficiency gains under FLOPs-matched comparisons, several limitations remain. First, a gap exists between theoretical complexity and practical throughput, as the patch aggregation kernels require further low-level optimization to match highly optimized fixed-length serving systems. Second, in our staged training pipeline, the boundary detector is pre-trained and frozen. While this approach to staged training is effective, learning patch boundaries jointly with the recommendation objective is a promising direction. Third, hardware constraints precluded study of ultra-long horizons (N > 500), where compression benefits may be even more pronounced. We hypothesize that richer multi-channel behavioral signals would provide more informative boundaries for contextual segmentation. More broadly, our contrastive-entropy detector is one effective instantiation of the boundary module, and identifying the best boundary-detection design (spanning alternative surprise signals, temporal cues, and training objectives) remains an open direction for future work.

\section{Conclusions}

In this study we introduced DP-Rec, a dynamic latent patching architecture with temporally adaptive boundary detection. DP-Rec couples item-level processing with explicit temporal modeling, combining temporal signals with contrastive entropy surprise to adaptively determine patch boundaries. This allows the model to overcome the quadratic constraints of standard Transformers while preserving high-fidelity representations of extended user histories. Experiments across three widely known public benchmarks confirm a superior computational efficiency–accuracy trade-off over both uncompressed and fixed-rate compression baselines, demonstrating that computation can follow the information density of user behavior rather than the arbitrary constraints of fixed boundaries. We invite future research to build upon this work to explore the potential of adaptive sequence modeling in recommender systems. \footnote{In preparing this paper, the authors utilized Gemini to improve the clarity, grammar, and presentation of the manuscript. No AI tools were used for model design, experimentation, or analysis; all technical contributions remain exclusively those of the authors.}

\begin{acks}
We sincerely thank Giri Iyengar and Bayan Bruss for their sponsorship, productive discussions, and critical reviews of this work.
\end{acks}


\bibliographystyle{acm}
\bibliography{biblio}

@String{Springer = "Springer-Verlag" }

@online{sasrectorch,
  author  = {Zan Huang},
  title   = {SASRec.pytorch},
  year    = {2020},
  url     = {https://github.com/pmixer/SASRec.pytorch}
}

@article{hoffmann2022training,
  title={Training compute-optimal large language models},
  author={Hoffmann, Jordan and Borgeaud, Sebastian and Mensch, Arthur and Buchatskaya, Elena and Cai, Trevor and Rutherford, Eliza and Casas, DDL and Hendricks, Lisa Anne and Welbl, Johannes and Clark, Aidan and others},
  journal={arXiv preprint arXiv:2203.15556},
  volume={10},
  year={2022}
}

@misc{rectools,
  author = {MobileTeleSystems},
  title = {RecTools: Library to build Recommendation Systems},
  year = {2024},
  publisher = {GitHub},
  journal = {GitHub repository},
  howpublished = {\url{https://github.com/MobileTeleSystems/RecTools}},
}

@article{gru4rec,
  title={Session-based recommendations with recurrent neural networks},
  author={Hidasi, Bal{\'a}zs and Karatzoglou, Alexandros and Baltrunas, Linas and Tikk, Domonkos},
  journal={arXiv preprint arXiv:1511.06939},
  year={2015}
}

@article{hstu,
  title={Actions speak louder than words: Trillion-parameter sequential transducers for generative recommendations},
  author={Zhai, Jiaqi and Liao, Lucy and Liu, Xing and Wang, Yueming and Li, Rui and Cao, Xuan and Gao, Leon and Gong, Zhaojie and Gu, Fangda and He, Michael and others},
  journal={arXiv preprint arXiv:2402.17152},
  year={2024}
}

@article{block,
  title={Block transformer: Global-to-local language modeling for fast inference},
  author={Ho, Namgyu and Bae, Sangmin and Kim, Taehyeon and Jo, Hyunjik and Kim, Yireun and Schuster, Tal and Fisch, Adam and Thorne, James and Yun, Se-Young},
  journal={Advances in Neural Information Processing Systems},
  volume={37},
  pages={48740--48783},
  year={2024}
}

@inproceedings{blt,
  title={Byte latent transformer: Patches scale better than tokens},
  author={Pagnoni, Artidoro and Pasunuru, Ramakanth and Rodriguez, Pedro and Nguyen, John and Muller, Benjamin and Li, Margaret and Zhou, Chunting and Yu, Lili and Weston, Jason E and Zettlemoyer, Luke and others},
  booktitle={Proceedings of the 63rd Annual Meeting of the Association for Computational Linguistics (Volume 1: Long Papers)},
  pages={9238--9258},
  year={2025}
}

@inproceedings{longer,
  title={Longer: Scaling up long sequence modeling in industrial recommenders},
  author={Chai, Zheng and Ren, Qin and Xiao, Xijun and Yang, Huizhi and Han, Bo and Zhang, Sijun and Chen, Di and Lu, Hui and Zhao, Wenlin and Yu, Lele and others},
  booktitle={Proceedings of the Nineteenth ACM Conference on Recommender Systems},
  pages={247--256},
  year={2025}
}

@article{survey2023,
  title={A survey on user behavior modeling in recommender systems},
  author={He, Zhicheng and Liu, Weiwen and Guo, Wei and Qin, Jiarui and Zhang, Yingxue and Hu, Yaochen and Tang, Ruiming},
  journal={arXiv preprint arXiv:2302.11087},
  year={2023}
}

@inproceedings{scalinglaw,
  title={Scaling law of large sequential recommendation models},
  author={Zhang, Gaowei and Hou, Yupeng and Lu, Hongyu and Chen, Yu and Zhao, Wayne Xin and Wen, Ji-Rong},
  booktitle={Proceedings of the 18th ACM Conference on Recommender Systems},
  pages={444--453},
  year={2024}
}

@inproceedings{scalingseq,
  title={Scaling sequential recommendation models with transformers},
  author={Zivic, Pablo and Vazquez, Hernan and S{\'a}nchez, Jorge},
  booktitle={Proceedings of the 47th international ACM SIGIR conference on research and development in information retrieval},
  pages={1567--1577},
  year={2024}
}

@article{scaling,
  title={Scaling new frontiers: Insights into large recommendation models},
  author={Guo, Wei and Wang, Hao and Zhang, Luankang and Chin, Jin Yao and Liu, Zhongzhou and Cheng, Kai and Pan, Qiushi and Lee, Yi Quan and Xue, Wanqi and Shen, Tingjia and others},
  journal={arXiv preprint arXiv:2412.00714},
  year={2024}
}

@article{rethinkseqrec,
  title={Rethinking Lifelong Sequential Recommendation with Incremental Multi-Interest Attention},
  author={YongjiWu, Lu Yin and Lian, Defu and Yin, Mingyang and Gong, Neil Zhenqiang and Zhou, Jingren and Yang, Hongxia},
  journal={arXiv preprint arXiv:2105.14060},
  year={2021}
}

@article{scalingpre,
  title={Scaling generative pre-training for user ad activity sequences},
  author={Chitlangia, Sharad and Kesari, Krishna Reddy and Agarwal, Rajat},
  year={2023}
}

@article{attn,
  title={Attention is all you need},
  author={Vaswani, Ashish and Shazeer, Noam and Parmar, Niki and Uszkoreit, Jakob and Jones, Llion and Gomez, Aidan N and Kaiser, {\L}ukasz and Polosukhin, Illia},
  journal={Advances in neural information processing systems},
  volume={30},
  year={2017}
}

@inproceedings{pinnerformer,
  title={Pinnerformer: Sequence modeling for user representation at pinterest},
  author={Pancha, Nikil and Zhai, Andrew and Leskovec, Jure and Rosenberg, Charles},
  booktitle={Proceedings of the 28th ACM SIGKDD conference on knowledge discovery and data mining},
  pages={3702--3712},
  year={2022}
}

@inproceedings{sum,
  title={Scaling User Modeling: Large-scale Online User Representations for Ads Personalization in Meta},
  author={Zhang, Wei and Li, Dai and Liang, Chen and Zhou, Fang and Zhang, Zhongke and Wang, Xuewei and Li, Ru and Zhou, Yi and Huang, Yaning and Liang, Dong and others},
  booktitle={Companion Proceedings of the ACM Web Conference 2024},
  pages={47--55},
  year={2024}
}

@article{dce,
  title={Dynamic Customer Embeddings for Financial Service Applications},
  author={Chitsazan, Nima and Sharpe, Samuel and Katariya, Dwipam and Cheng, Qianyu and Rajasethupathy, Karthik},
  journal={arXiv preprint arXiv:2106.11880},
  year={2021}
}

@inproceedings{twin,
  title={TWIN: TWo-stage interest network for lifelong user behavior modeling in CTR prediction at kuaishou},
  author={Chang, Jianxin and Zhang, Chenbin and Fu, Zhiyi and Zang, Xiaoxue and Guan, Lin and Lu, Jing and Hui, Yiqun and Leng, Dewei and Niu, Yanan and Song, Yang and others},
  booktitle={Proceedings of the 29th ACM SIGKDD Conference on Knowledge Discovery and Data Mining},
  pages={3785--3794},
  year={2023}
}

@inproceedings{practice,
  title={Practice on long sequential user behavior modeling for click-through rate prediction},
  author={Pi, Qi and Bian, Weijie and Zhou, Guorui and Zhu, Xiaoqiang and Gai, Kun},
  booktitle={Proceedings of the 25th ACM SIGKDD international conference on knowledge discovery \& data mining},
  pages={2671--2679},
  year={2019}
}

@inproceedings{lmnr,
  title={Large Memory Network for Recommendation},
  author={Lu, Hui and Chai, Zheng and Zheng, Yuchao and Chen, Zhe and Xie, Deping and Xu, Peng and Zhou, Xun and Wu, Di},
  booktitle={Companion Proceedings of the ACM on Web Conference 2025},
  pages={1162--1166},
  year={2025}
}

@article{marm,
  title={Marm: Unlocking the future of recommendation systems through memory augmentation and scalable complexity},
  author={Lv, Xiao and Cao, Jiangxia and Guan, Shijie and Zhou, Xiaoyou and Qi, Zhiguang and Zang, Yaqiang and Li, Ming and Wang, Ben and Gai, Kun and Zhou, Guorui},
  journal={arXiv preprint arXiv:2411.09425},
  year={2024}
}

@inproceedings{coseRenn,
  title={Contextual and sequential user embeddings for large-scale music recommendation},
  author={Hansen, Casper and Hansen, Christian and Maystre, Lucas and Mehrotra, Rishabh and Brost, Brian and Tomasi, Federico and Lalmas, Mounia},
  booktitle={Proceedings of the 14th ACM Conference on Recommender Systems},
  pages={53--62},
  year={2020}
}

@article{fintrec,
  title={FinTRec: Transformer Based Unified Contextual Ads Targeting and Personalization for Financial Applications},
  author={Katariya, Dwipam and Varma, Snehita and Shreemali, Akshat and Wu, Benjamin and Mishra, Kalanand and Mohanty, Pranab},
  journal={arXiv preprint arXiv:2511.14865},
  year={2025}
}

@article{timesync,
  title={Timesync: Temporal intent modelling with synchronized context encodings for financial service applications},
  author={Katariya, Dwipam and Origgi, Juan Manuel and Wang, Yage and Caputo, Thomas},
  journal={arXiv preprint arXiv:2410.12825},
  year={2024}
}

@article{rnn4rec,
  title={Session-based recommendations with recurrent neural networks},
  author={Hidasi, Bal{\'a}zs and Karatzoglou, Alexandros and Baltrunas, Linas and Tikk, Domonkos},
  journal={arXiv preprint arXiv:1511.06939},
  year={2015}
}

@inproceedings{sasrec,
  title={Self-attentive sequential recommendation},
  author={Kang, Wang-Cheng and McAuley, Julian},
  booktitle={2018 IEEE international conference on data mining (ICDM)},
  pages={197--206},
  year={2018},
  organization={IEEE}
}

@inproceedings{bert4rec,
  title={BERT4Rec: Sequential recommendation with bidirectional encoder representations from transformer},
  author={Sun, Fei and Liu, Jun and Wu, Jian and Pei, Changhua and Lin, Xiao and Ou, Wenwu and Jiang, Peng},
  booktitle={Proceedings of the 28th ACM international conference on information and knowledge management},
  pages={1441--1450},
  year={2019}
}

@inproceedings{fpmc,
  title={Factorizing personalized markov chains for next-basket recommendation},
  author={Rendle, Steffen and Freudenthaler, Christoph and Schmidt-Thieme, Lars},
  booktitle={Proceedings of the 19th international conference on World wide web},
  pages={811--820},
  year={2010}
}

@article{bpr,
  title={BPR: Bayesian personalized ranking from implicit feedback},
  author={Rendle, Steffen and Freudenthaler, Christoph and Gantner, Zeno and Schmidt-Thieme, Lars},
  journal={arXiv preprint arXiv:1205.2618},
  year={2012}
}

@inproceedings{ncf,
  title={Neural collaborative filtering},
  author={He, Xiangnan and Liao, Lizi and Zhang, Hanwang and Nie, Liqiang and Hu, Xia and Chua, Tat-Seng},
  booktitle={Proceedings of the 26th international conference on world wide web},
  pages={173--182},
  year={2017}
}

@article{patchtst,
  title={A time series is worth 64 words: Long-term forecasting with transformers},
  author={Nie, Yuqi and Nguyen, Nam H and Sinthong, Phanwadee and Kalagnanam, Jayant},
  journal={arXiv preprint arXiv:2211.14730},
  year={2022}
}

@article{mspatch,
  title={MSPatch: A multi-scale patch mixing framework for multivariate time series forecasting},
  author={Cao, Yizhi and Tian, Zijian and Guo, Wenjie and Liu, Xinggao},
  journal={Expert Systems with Applications},
  volume={273},
  pages={126849},
  year={2025},
  publisher={Elsevier}
}

@inproceedings{hdmixer,
  title={Hdmixer: Hierarchical dependency with extendable patch for multivariate time series forecasting},
  author={Huang, Qihe and Shen, Lei and Zhang, Ruixin and Cheng, Jiahuan and Ding, Shouhong and Zhou, Zhengyang and Wang, Yang},
  booktitle={Proceedings of the AAAI conference on artificial intelligence},
  volume={38},
  number={11},
  pages={12608--12616},
  year={2024}
}

@inproceedings{morai,
  title={Unified training of universal time series forecasting transformers},
  author={Woo, Gerald and Liu, Chenghao and Kumar, Akshat and Xiong, Caiming and Savarese, Silvio and Sahoo, Doyen},
  booktitle={Forty-first International Conference on Machine Learning},
  year={2024}
}

@inproceedings{adapatch,
  title={AdaPatch: Adaptive Patch-Level Modeling for Non-Stationary Time Series Forecasting},
  author={Liu, Kun and Duan, Zhongjie and Chen, Cen and Wang, Yanhao and Cheng, Dawei and Liang, Yuqi},
  booktitle={Proceedings of the 34th ACM International Conference on Information and Knowledge Management},
  pages={1882--1891},
  year={2025}
}

@inproceedings{apn,
  title={Rethinking irregular time series forecasting: A simple yet effective baseline},
  author={Liu, Xvyuan and Qiu, Xiangfei and Wu, Xingjian and Li, Zhengyu and Guo, Chenjuan and Hu, Jilin and Yang, Bin},
  booktitle={Proceedings of the AAAI Conference on Artificial Intelligence},
  volume={40},
  number={28},
  pages={23873--23881},
  year={2026}
}

@article{rope,
  title={Roformer: Enhanced transformer with rotary position embedding},
  author={Su, Jianlin and Ahmed, Murtadha and Lu, Yu and Pan, Shengfeng and Bo, Wen and Liu, Yunfeng},
  journal={Neurocomputing},
  volume={568},
  pages={127063},
  year={2024},
  publisher={Elsevier}
}

@article{srnet,
  title={Enhancing time series forecasting through selective representation spaces: A patch perspective},
  author={Wu, Xingjian and Qiu, Xiangfei and Cheng, Hanyin and Li, Zhengyu and Hu, Jilin and Guo, Chenjuan and Yang, Bin},
  journal={arXiv preprint arXiv:2510.14510},
  year={2025}
}

@inproceedings{timemosaic,
  title={Timemosaic: Temporal heterogeneity guided time series forecasting via adaptive granularity patch and segment-wise decoding},
  author={Ding, Kuiye and Fan, Fanda and Hou, Chunyi and Wang, Zheya and Wang, Lei and Yang, Zhengxin and Zhan, Jianfeng},
  booktitle={Proceedings of the AAAI Conference on Artificial Intelligence},
  volume={40},
  number={25},
  pages={20790--20798},
  year={2026}
}

@article{entrope,
  title={EntroPE: Entropy-Guided Dynamic Patch Encoder for Time Series Forecasting},
  author={Abeywickrama, Sachith and Eldele, Emadeldeen and Wu, Min and Li, Xiaoli and Yuen, Chau},
  journal={arXiv preprint arXiv:2509.26157},
  year={2025}
}

@inproceedings{perciever,
  title={Perceiver: General perception with iterative attention},
  author={Jaegle, Andrew and Gimeno, Felix and Brock, Andy and Vinyals, Oriol and Zisserman, Andrew and Carreira, Joao},
  booktitle={International conference on machine learning},
  pages={4651--4664},
  year={2021},
  organization={PMLR}
}

@article{survey2026,
  title={A Survey of User Lifelong Behavior Modeling: Perspectives on Efficiency and Effectiveness},
  author={Zhou, Rui and Jia, Qinglin and Chen, Bo and Xu, Peng and Sun, Yijia and Lou, Siyuan and Fu, Chaoxin and Fu, Mengyuan and Shen, Guoming and Zhou, Zheli and others},
  year={2026},
  publisher={Preprints}
}

@article{onerec,
  title={Onerec: Unifying retrieve and rank with generative recommender and iterative preference alignment},
  author={Deng, Jiaxin and Wang, Shiyao and Cai, Kuo and Ren, Lejian and Hu, Qigen and Ding, Weifeng and Luo, Qiang and Zhou, Guorui},
  journal={arXiv preprint arXiv:2502.18965},
  year={2025}
}

@article{llmcdsr,
  title={Llmcdsr: Enhancing cross-domain sequential recommendation with large language models},
  author={Xin, Haoran and Sun, Ying and Wang, Chao and Xiong, Hui},
  journal={ACM Transactions on Information Systems},
  volume={43},
  number={5},
  pages={1--33},
  year={2025},
  publisher={ACM New York, NY}
}

@inproceedings{skilling,
  title={Killing two birds with one stone: Unifying retrieval and ranking with a single generative recommendation model},
  author={Zhang, Luankang and Song, Kenan and Lee, Yi Quan and Guo, Wei and Wang, Hao and Li, Yawen and Guo, Huifeng and Liu, Yong and Lian, Defu and Chen, Enhong},
  booktitle={Proceedings of the 48th International ACM SIGIR Conference on Research and Development in Information Retrieval},
  pages={2224--2234},
  year={2025}
}

@inproceedings{transact,
  title={Transact: Transformer-based realtime user action model for recommendation at pinterest},
  author={Xia, Xue and Eksombatchai, Pong and Pancha, Nikil and Badani, Dhruvil Deven and Wang, Po-Wei and Gu, Neng and Joshi, Saurabh Vishwas and Farahpour, Nazanin and Zhang, Zhiyuan and Zhai, Andrew},
  booktitle={Proceedings of the 29th ACM SIGKDD Conference on Knowledge Discovery and Data Mining},
  pages={5249--5259},
  year={2023}
}

@inproceedings{transactv2,
  title={TransAct V2: Lifelong User Action Sequence Modeling on Pinterest Recommendation},
  author={Xia, Xue and Joshi, Saurabh and Rajesh, Kousik and Li, Kangnan and Lu, Yangyi and Pancha, Nikil and Badani, Dhruvil and Xu, Jiajing and Eksombatchai, Pong},
  booktitle={Proceedings of the 34th ACM International Conference on Information and Knowledge Management},
  pages={6881--6882},
  year={2025}
}

@article{make,
  title={Make It Long, Keep It Fast: End-to-End 10k-Sequence Modeling at Billion Scale on Douyin},
  author={Guan, Lin and Yang, Jia-Qi and Zhao, Zhishan and Zhang, Beichuan and Sun, Bo and Luo, Xuanyuan and Ni, Jinan and Li, Xiaowen and Qi, Yuhang and Fan, Zhifang and others},
  journal={arXiv preprint arXiv:2511.06077},
  year={2025}
}

@inproceedings{dstn,
  title={Deep spatio-temporal neural networks for click-through rate prediction},
  author={Ouyang, Wentao and Zhang, Xiuwu and Li, Li and Zou, Heng and Xing, Xin and Liu, Zhaojie and Du, Yanlong},
  booktitle={Proceedings of the 25th ACM SIGKDD International Conference on Knowledge Discovery \& Data Mining},
  pages={2078--2086},
  year={2019}
}

@inproceedings{dfn,
  title={Deep feedback network for recommendation},
  author={Xie, Ruobing and Ling, Cheng and Wang, Yalong and Wang, Rui and Xia, Feng and Lin, Leyu},
  booktitle={Proceedings of the twenty-ninth international conference on international joint conferences on artificial intelligence},
  pages={2519--2525},
  year={2021}
}

@inproceedings{dien,
  title={Deep interest evolution network for click-through rate prediction},
  author={Zhou, Guorui and Mou, Na and Fan, Ying and Pi, Qi and Bian, Weijie and Zhou, Chang and Zhu, Xiaoqiang and Gai, Kun},
  booktitle={Proceedings of the AAAI conference on artificial intelligence},
  volume={33},
  number={01},
  pages={5941--5948},
  year={2019}
}

@inproceedings{xdm,
  title={Xdm: Improving sequential deep matching with unclicked user behaviors for recommender system},
  author={Lv, Fuyu and Li, Mengxue and Guo, Tonglei and Yu, Changlong and Sun, Fei and Jin, Taiwei and Ng, Wilfred},
  booktitle={International Conference on Database Systems for Advanced Applications},
  pages={364--376},
  year={2022},
  organization={Springer}
}

@inproceedings{racp,
  title={Modeling users' contextualized page-wise feedback for click-through rate prediction in e-commerce search},
  author={Fan, Zhifang and Ou, Dan and Gu, Yulong and Fu, Bairan and Li, Xiang and Bao, Wentian and Dai, Xin-Yu and Zeng, Xiaoyi and Zhuang, Tao and Liu, Qingwen},
  booktitle={Proceedings of the fifteenth ACM international conference on web search and data mining},
  pages={262--270},
  year={2022}
}

@article{temp4ctr,
  title={Time-aligned Exposure-enhanced Model for Click-Through Rate Prediction},
  author={Zhang, Hengyu and Meng, Chang and Guo, Wei and Guo, Huifeng and Zhu, Jieming and Zhao, Guangpeng and Tang, Ruiming and Li, Xiu},
  journal={arXiv preprint arXiv:2308.09966},
  year={2023}
}

@inproceedings{sim,
  title={Search-based user interest modeling with lifelong sequential behavior data for click-through rate prediction},
  author={Pi, Qi and Zhou, Guorui and Zhang, Yujing and Wang, Zhe and Ren, Lejian and Fan, Ying and Zhu, Xiaoqiang and Gai, Kun},
  booktitle={Proceedings of the 29th ACM International Conference on Information \& Knowledge Management},
  pages={2685--2692},
  year={2020}
}

@inproceedings{ubr4ctr,
  title={User behavior retrieval for click-through rate prediction},
  author={Qin, Jiarui and Zhang, Weinan and Wu, Xin and Jin, Jiarui and Fang, Yuchen and Yu, Yong},
  booktitle={Proceedings of the 43rd international ACM SIGIR conference on research and development in information retrieval},
  pages={2347--2356},
  year={2020}
}

@inproceedings{twinv2,
  title={Twin v2: Scaling ultra-long user behavior sequence modeling for enhanced ctr prediction at kuaishou},
  author={Si, Zihua and Guan, Lin and Sun, ZhongXiang and Zang, Xiaoxue and Lu, Jing and Hui, Yiqun and Cao, Xingchao and Yang, Zeyu and Zheng, Yichen and Leng, Dewei and others},
  booktitle={Proceedings of the 33rd ACM International Conference on Information and Knowledge Management},
  pages={4890--4897},
  year={2024}
}

@article{recmamba,
  title={Uncovering selective state space model's capabilities in lifelong sequential recommendation},
  author={Yang, Jiyuan and Li, Yuanzi and Zhao, Jingyu and Wang, Hanbing and Ma, Muyang and Ma, Jun and Ren, Zhaochun and Zhang, Mengqi and Xin, Xin and Chen, Zhumin and others},
  journal={arXiv preprint arXiv:2403.16371},
  year={2024}
}

@inproceedings{personexprt,
  title={Efficient Sequential Recommendation for Long Term User Interest Via Personalization},
  author={Zhang, Qiang and Yu, Hanchao and Ji, Ivan and Yuan, Chen and Zhang, Yi and Liu, Chihuang and Wang, Xiaolong and Lambert, Christopher E and Chen, Ren and Kovacs, Chen and others},
  booktitle={2025 IEEE International Conference on Data Mining (ICDM)},
  pages={913--922},
  year={2025},
  organization={IEEE}
}

@article{ml1m,
  title={The movielens datasets: History and context},
  author={Harper, F Maxwell and Konstan, Joseph A},
  journal={Acm transactions on interactive intelligent systems (tiis)},
  volume={5},
  number={4},
  pages={1--19},
  year={2015},
  publisher={Acm New York, NY, USA}
}

@inproceedings{kuairand,
  title={Kuairand: An unbiased sequential recommendation dataset with randomly exposed videos},
  author={Gao, Chongming and Li, Shijun and Zhang, Yuan and Chen, Jiawei and Li, Biao and Lei, Wenqiang and Jiang, Peng and He, Xiangnan},
  booktitle={Proceedings of the 31st ACM international conference on information \& knowledge management},
  pages={3953--3957},
  year={2022}
}

@article{reinpatch,
  title={Dynamic Tokenization via Reinforcement Patching: End-to-end Training and Zero-shot Transfer},
  author={Wu, Yulun and Ankireddy, Sravan Kumar and Sharpe, Samuel and Seleznev, Nikita and Yuan, Dehao and Kim, Hyeji and Nguyen, Nam H},
  journal={arXiv preprint arXiv:2603.26097},
  year={2026}
}

@inproceedings{
srsnet,
title={Enhancing Time Series Forecasting through Selective Representation Spaces: A Patch Perspective},
author={Xingjian Wu and Xiangfei Qiu and Hanyin Cheng and Zhengyu Li and Jilin Hu and Chenjuan Guo and Bin Yang},
booktitle={The Thirty-ninth Annual Conference on Neural Information Processing Systems},
year={2026},
url={https://openreview.net/forum?id=BirE0jYKt0}
}

\appendix

\section{Relevant Work}
\label{sec:relevant_work}
\subsection{Long Sequential Recommendation}
Sequential recommendation has attracted substantial attention in both academia and industry \cite{fintrec, pinnerformer, coseRenn, onerec, llmcdsr, skilling, transactv2, make, transact, hstu}. Transformer based models, especially SASRec \cite{sasrec} and BERT4Rec \cite{bert4rec}, have become foundational in this area by demonstrating the effectiveness of self attention for modeling user behavior sequences. These methods significantly advanced the state of the art over earlier sequential models such as GRU4Rec \cite{rnn4rec}, FPMC \cite{fpmc}, and CASER, as well as non sequential approaches including BPR-MF \cite{bpr} and NCF \cite{ncf}. Nevertheless, most existing methods are designed for relatively short or truncated interaction histories. Scaling these architectures to lifelong user logs is restricted not only by the quadratic complexity of self-attention but also by the increasing heterogeneity of user interests. These factors render assumptions of signal homogeneity and limited temporal scope increasingly unrealistic in large-scale settings \cite{survey2026}. Early work in this area focused primarily on explicit feedback, such as numerical ratings. Subsequent studies, however, established that Click-Through Rate (CTR) prediction—the estimation of a user's probability of engagement—benefits from jointly modeling heterogeneous behavioral signals. Specifically, these models incorporate both positive signals, such as successful clicks, and negative signals, such as unclicked impressions where an item was exposed to the user but failed to elicit an interaction. In particular, DSTN \cite{dstn}, DFN \cite{dfn}, DIEN \cite{dien}, and XDM \cite{xdm} demonstrated the importance of incorporating such multi-feedback information. Because implicit signals such as views and impressions are substantially more abundant than clicks, their inclusion greatly increases sequence length and makes full long-range sequence modeling computationally infeasible. To better exploit richer behavioral information, RACP \cite{racp} introduced a hierarchical attention architecture that captures both intra-page dependencies among co displayed items and inter-page interest transitions across sessions. Despite these advances, prior studies largely overlook the impact of sequence truncation in multi-feedback settings, where truncation can weaken cross-signal dependencies and limit the joint modeling of different feedback types. TEM4CTR \cite{temp4ctr} addresses this issue through a three-stage framework that first retrieves unclicked records near clicked ones, then applies attention to identify relevant exposure information, and finally jointly extracts latent user interests from clicked and unclicked sequences for CTR prediction. Hence, recent practice increasingly favors integrating a broader range of user behavioral signals into recommendation models. While this enriches user representation, it also leads to substantial sequence length growth, thereby making efficient sequence processing a critical requirement.

\subsection{Sequence Compression in RecSys}  
Under strict inference budget constraints, industrial RecSys often rely on staged compression pipelines to make long behavioral histories tractable. These include two stage retrieval methods such as TWIN \cite{twin}, TWIN V2 \cite{twinv2}, and SIM \cite{sim}, pre trained embedding approaches such as PINNERFORMER \cite{pinnerformer}, SUM \cite{sum}, DCE \cite{dce}, CoSeRNN \cite{coseRenn}, and FinTRec \cite{fintrec}, as well as memory augmented methods such as LMN \cite{lmnr}, Marm \cite{marm}, MIMN \cite{practice}. While these methods differ in design, they share a common goal of reducing sequence length before downstream modeling. For example, TWIN V2 \cite{twinv2} compresses histories through clustering, whereas TWIN \cite{twin}, UBR4CTR \cite{ubr4ctr}, SIM \cite{sim}, and TransAct \cite{transact} retrieve a target relevant subset of interactions followed by lightweight self attention. Although effective in reducing computation, these strategies often sacrifice signal fidelity. Because compression is performed upstream, they are typically not optimized end-to-end with downstream item relevance, and may therefore miss important real-time behavioral signals \cite{transactv2}. This separation between compression and recommendation also limits the model’s ability to preserve the fine-grained structure of long behavioral histories. In addition, staged designs introduce operational overhead by requiring coordination between upstream retraining and downstream serving. Recent work has attempted to address these limitations more directly.  TransAct V2 \cite{transactv2} demonstrates the benefit of combining lifelong histories with real time signals, but restricts the lifelong component to high signal explicit interactions, keeping the final sequence length at approximately \(10^2\) and still discarding fine-grained information. Another direction explores generative architectures such as HSTU \cite{hstu}, though these often require extremely large models and substantial compute budgets \cite{transact}. LONGER \cite{longer} reduces computation through fixed window token merging followed by a lightweight Transformer over adjacent tokens. While this lowers FLOPs without hurting baseline accuracy, fixed-rate merging cannot adapt to the uneven information density of heterogeneous user behaviors. Recently, PerRec \cite{personexprt} introduced learnable tokens for sequence compression; however, it still relied on a fixed set of learnable tokens. In contrast, DP-Rec processes raw long-range sequences directly with Transformer architecture, reducing reliance on aggressive upfront compression and fixed set of assumed windows and/or learnable tokens.

\subsection{Sequence Compressions in Other Domains}
Recent research has increasingly focused on relaxing fixed-length constraints through adaptive or multi-scale segmentation strategies. Frameworks such as MSPatch \cite{mspatch} and HDMixer \cite{hdmixer} leverage multi-resolution patching and extendable interpolation, respectively, while MOIRAI \cite{morai} determines patch granularity within the frequency domain. Other approaches address sequence non-stationarity and irregular intervals; for example, AdaPatch \cite{adapatch} introduces patch-level normalization, and APN \cite{apn} employs time-aware aggregation driven by data-density regularization. More granular architectures like SRSNet \cite{srnet} and TimeMosaic \cite{timemosaic} utilize selective reassembly or motif-based adaptation to enhance representation efficiency and decoding specialization. EntroPE\cite{entrope}, recently explored application of BLT to timeseries demonstrating strong efficiency and performance gains.

\section{Temporal Rotary Positional Embeddings (Time-RoPE)}                                      
\label{sec:timerope}                   
To encode interaction rhythms, we adapt Rotary Positional Embeddings \cite{rope} to absolute timestamps $T_t$. For a latent vector $\mathbf{x} \in \mathbb{R}^d$, the transformation $\mathcal{R}(\mathbf{x}, T_t)$ rotates $d/2$ dimension pairs in the query-key space:
\begin{equation}
\mathcal{R}(\mathbf{x}, T_t) = \mathbf{x} \odot \cos(\mathbf{\Theta} T_t) + \mathbf{\tilde{x}} \odot \sin(\mathbf{\Theta} T_t)
\end{equation}
where $\mathbf{\Theta} = \{\theta_i = 10000^{-2i/d}\}_{i=1}^{d/2}$ are the inverse frequencies. This makes the attention score between two interactions at times $T_t$ and $T_{t'}$ a function of their time lag $(T_t - T_{t'})$, giving both the local and latent Transformers a continuous, density-aware sense of behavioral timing.

\section{Lemma - Patch-count calibration under distinct scores}
  The following lemma formalizes the calibration property of the proposed thresholding rule under an idealized no-tie condition. Specifically, it shows that when all valid surprise scores are distinct, selecting the threshold as the $k$-th largest valid score yields exactly the intended number of detected boundaries, and hence matches the target average patch budget up to rounding. The result is purely combinatorial and serves to justify the thresholding mechanism, rather than the semantic quality of the resulting boundaries.

  \begin{lemma}[Patch-count calibration under distinct scores]
  \label{lem:threshold_count}
  Let $\phi \in \mathbb{R}^{B \times N}$ denote the surprise scores, $\boldsymbol{\mu} \in \{0,1\}^{B \times N}$ the padding mask, and $\Phi=\{\phi_{u,t}\mid \boldsymbol{\mu}_{u,t}=1\}$ the multiset of valid
  scores. Let $\hat{M}$ be the target average number of patches, $\bar{N} = |\Phi|/B$ the average sequence length, and $M_{\mathrm{eff}} = \min(\hat{M}, \bar{N})$. Define the boundary budget as:
  \[
  k = \mathrm{round}((M_{\mathrm{eff}} - 1)B).
  \]
  Assume all values in $\Phi$ are distinct and $0 \le k \le |\Phi|$. Let $\tau$ be the $k$-th largest element of $\Phi$, and $\boldsymbol{\Gamma}_{u,t} = \mathbb{I}[\phi_{u,t} \ge \tau] \cdot
  \boldsymbol{\mu}_{u,t}$. Then the total number of detected boundaries is exactly $k$, and the total number of patches is:
  \[
  \text{Total Patches} = k + B = \mathrm{round}(M_{\mathrm{eff}} B).
  \]
  Consequently, the average number of patches per sequence equals $M_{\mathrm{eff}}$.
  \end{lemma}

  \begin{proof}
  Since all scores in $\Phi$ are distinct, the $k$-th largest value $\tau$ is unique. By definition of the $k$-th order statistic, there are exactly $k$ elements in $\Phi$ with $\phi_{u,t} \ge \tau$. Thus the
  total boundary count is:
  \[
  \sum_{u=1}^{B} \sum_{t=1}^{N} \boldsymbol{\Gamma}_{u,t} = k.
  \]
  Each sequence $u \in \{1, \dots, B\}$ is partitioned into $M_u = \gamma_u + 1$ patches, where $\gamma_u = \sum_t \boldsymbol{\Gamma}_{u,t}$ is its boundary count. The total number of patches across the batch is:
  \[
  \sum_{u=1}^{B} M_u = \sum_{u=1}^{B} (\gamma_u + 1) = \left( \sum_{u=1}^{B} \gamma_u \right) + B = k + B.
  \]
  Substituting $k = \mathrm{round}((M_{\mathrm{eff}} - 1)B) = \mathrm{round}(M_{\mathrm{eff}} B - B)$ and using that $B$ is an integer:
  \[
  k + B = \mathrm{round}(M_{\mathrm{eff}} B) - B + B = \mathrm{round}(M_{\mathrm{eff}} B).
  \]
  Dividing by $B$ yields the average patch count $M_{\mathrm{eff}}$, which equals the target $\hat{M}$ whenever the average sequence length $\bar{N}$ is sufficient.
  \end{proof}

\section{Robustness to Item Popularity Tiers}
\label{sec:longtail}
\begin{figure*}[t]
     \centering
     \begin{subfigure}[b]{0.32\textwidth}
          \centering
          \includegraphics[width=\textwidth]{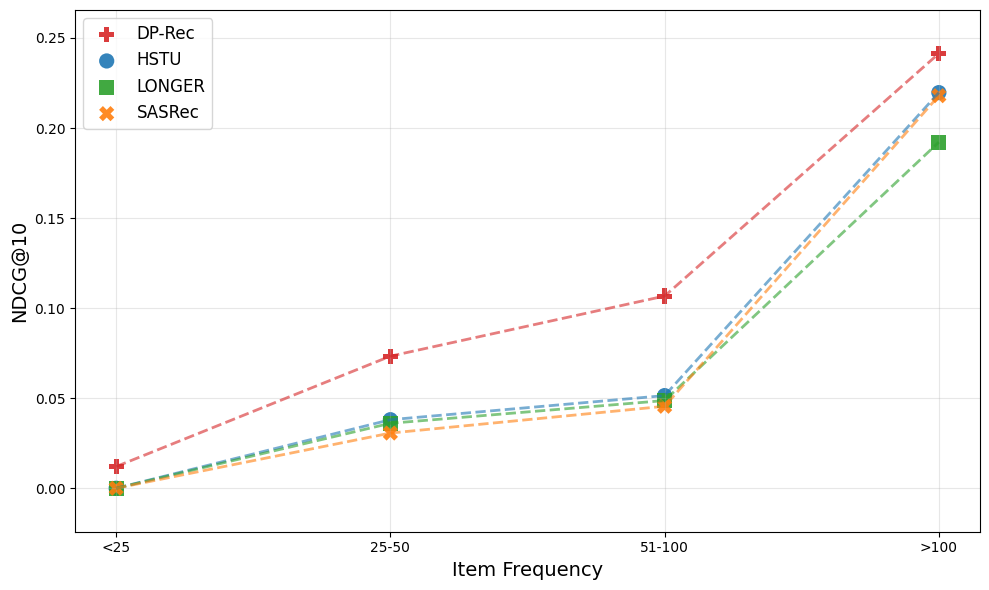}
          \caption{ML-1M}
     \end{subfigure}%
    \hfill
     \begin{subfigure}[b]{0.32\textwidth}
          \centering
          \includegraphics[width=\textwidth]{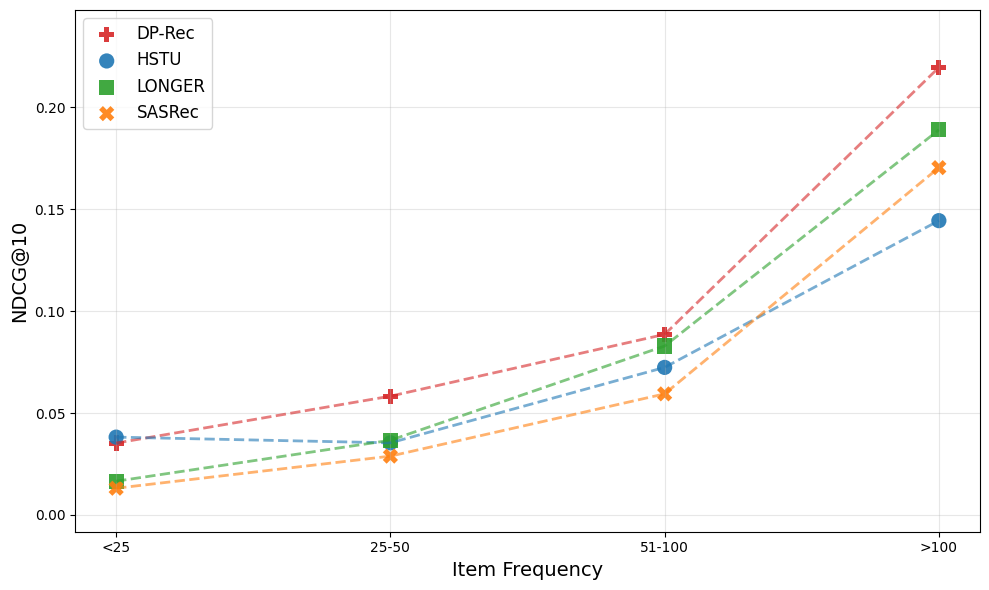}
          \caption{ML-10M-LONG}
     \end{subfigure}%
     \hfill 
     \begin{subfigure}[b]{0.32\textwidth}
          \centering
          \includegraphics[width=\textwidth]{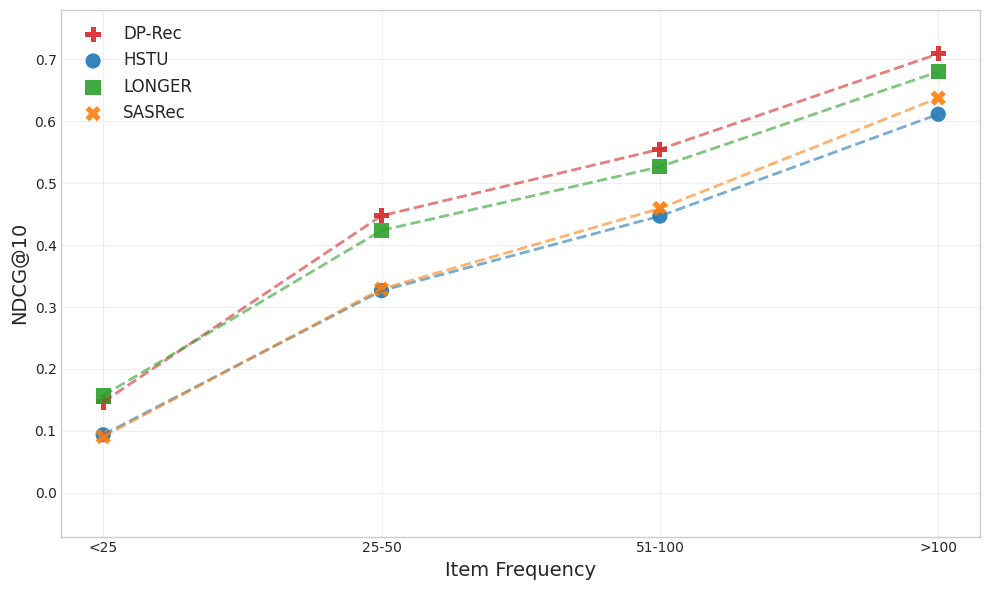}
          \caption{KuaiRand}
     \end{subfigure}
     \caption{Recommendation Quality (NDCG@10) by Item Popularity across three benchmark datasets.}
     \label{fig:item_popularity}
\end{figure*}

The analysis of recommendation utility across item frequency tiers reveals that while all models benefit from increased interaction density, their ability to handle "tail" versus "head" items varies significantly. Across all three datasets, we observe a monotonic increase in NDCG@10 as item frequency rises, confirming that higher-resolution signals for popular items naturally facilitate more accurate predictions. \textbf{DP-Rec} consistently establishes the performance ceiling across the entire popularity spectrum. Specifically, in low-frequency bins ($<25$ interactions), \textbf{DP-Rec} exhibits greater resilience than standard baselines, suggesting that its adaptive boundary detection effectively extracts meaningful patterns even from sparse interaction data where fixed-horizon models like \textit{SASRec} and \textit{HSTU} struggle to isolate a clear signal.

In the high-frequency regime ($>100$ interactions), the advantage of compression-based architectures becomes even more pronounced. In \textit{ML-10M-Long} and \textit{KuaiRand}, both \textbf{DP-Rec} and \textit{LONGER} significantly outperform fixed-horizon baselines. This indicates that for popular items, relevant behavioral context is likely distributed across a much longer temporal window; while fixed-horizon models are limited by rigid truncation, the adaptive patching of \textbf{DP-Rec} allows it to aggregate information from the entire user history without signal dilution. Notably, in \textit{KuaiRand}, \textit{LONGER} performs exceptionally well, closely trailing \textbf{DP-Rec}, whereas in \textit{ML-1M}, \textbf{DP-Rec} maintains a wider margin over all competitors in the mid-to-head tiers. Collectively, these results demonstrate that \textbf{DP-Rec} provides a robust architectural advantage that scales effectively from sparse tail items to dense head items, outperforming both traditional attention-based models and uniform compression strategies.

\section{Performance across Temporal Lag Bins}
\begin{figure*}[t]
     \centering
     \begin{subfigure}[b]{0.63\columnwidth}
          \centering
          \includegraphics[width=\textwidth]{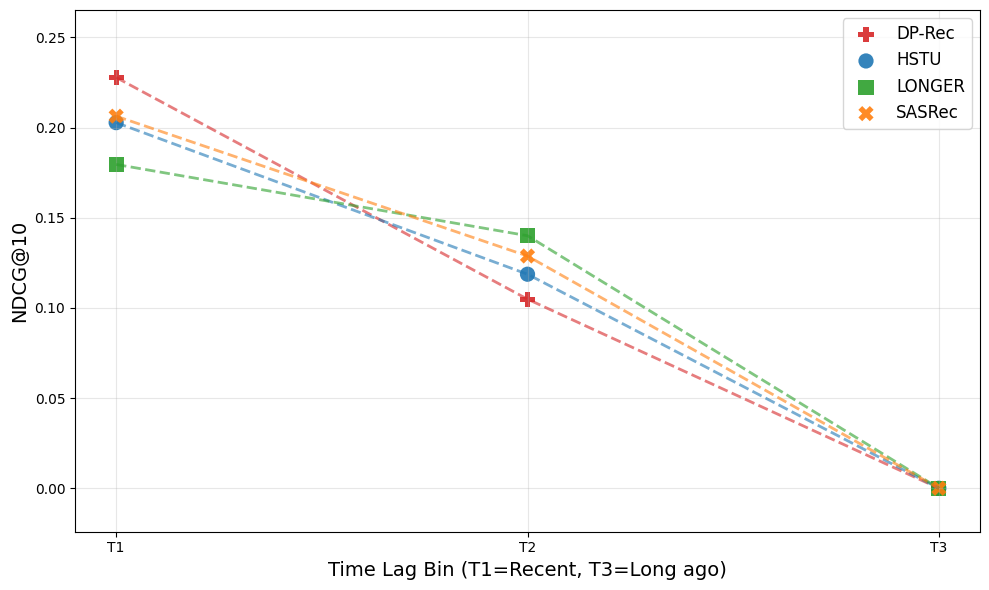}
          \caption{ML-1M}
    \end{subfigure}
    \begin{subfigure}[b]{0.63\columnwidth}
          \centering
          \includegraphics[width=\textwidth]{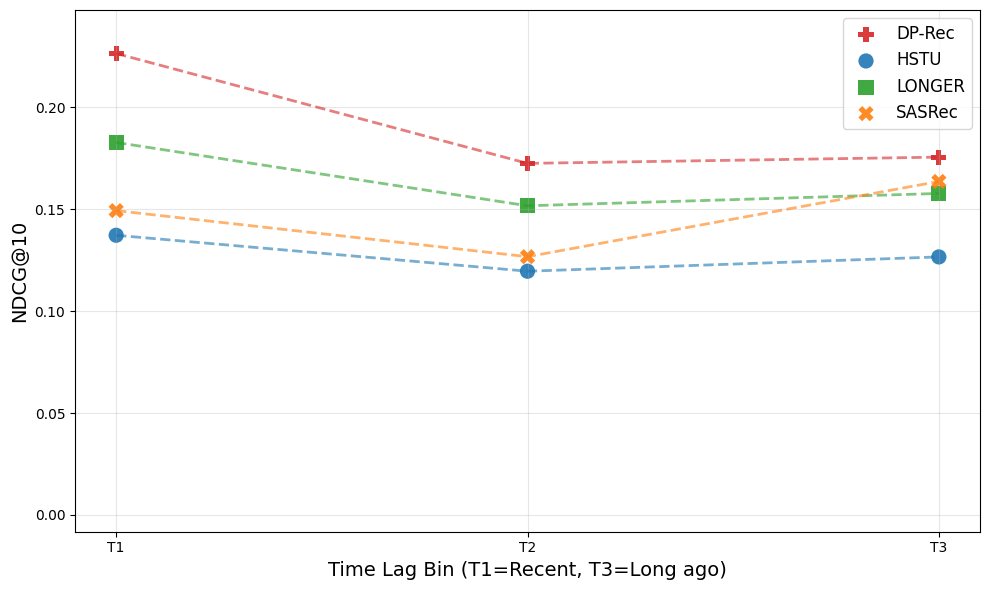}
          \caption{ML-10M-LONG}
    \end{subfigure}
    \begin{subfigure}[b]{0.63\columnwidth}
          \centering
          \includegraphics[width=\textwidth]{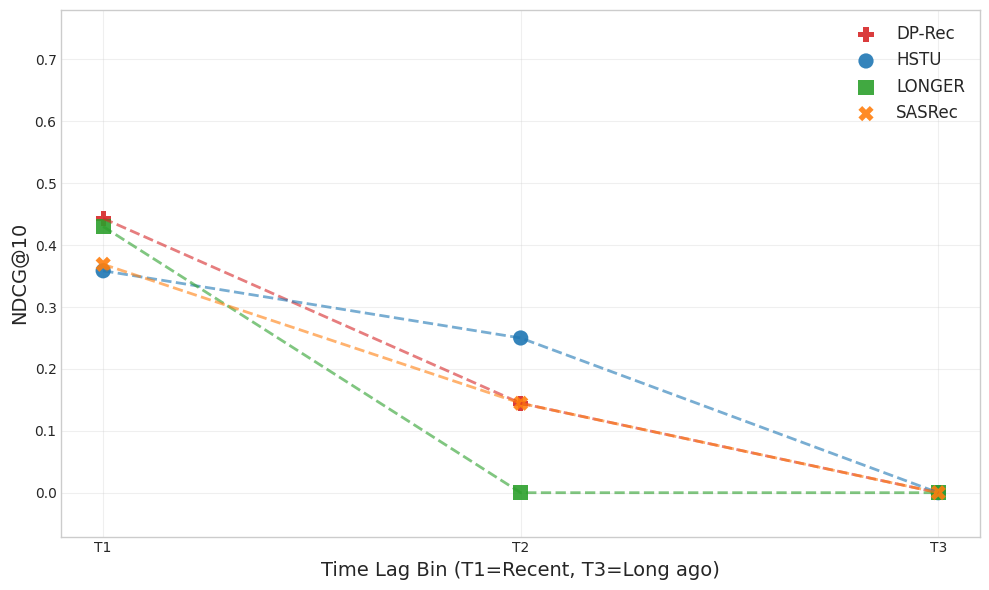}
          \caption{KuaiRand}
     \end{subfigure}
     \caption{Recommendation Quality(NDCG@10) by Time Lag between the last item and the item in the test to be ranked.}
     \label{fig:temporal_lag}
\end{figure*}
Temporal analysis across the experimental suite clarifies how the distribution of target signals modulates model performance over time. In \textit{ML-1M} and \textit{KuaiRand}, we observe a significant recency bias, where predictive utility for all models effectively collapses by the most distal bin (T3). In these scenarios, \textbf{DP-Rec} demonstrates superior peak performance in the immediate temporal bin (T1), indicating its efficacy in capturing high-entropy signals within recent interactions. However, model behavior diverges in the transition to T2: while \textit{LONGER} maintains a slight edge in \textit{ML-1M}, it suffers a catastrophic collapse to zero in \textit{KuaiRand}, whereas \textit{HSTU} exhibits higher middle-range robustness. This suggests that while uniform compression or fixed horizons can occasionally capture specific temporal windows, they are inherently prone to failure when the underlying signal density shifts.
The results for \textit{ML-10M-Long} provide the most compelling evidence for \textbf{DP-Rec}'s architectural advantage in long-range scenarios. Unlike the other datasets, \textit{ML-10M-Long} exhibits remarkable signal persistence, with significant utility maintained even in the T3 bin. Here, \textbf{DP-Rec} consistently defines the performance ceiling across the entire temporal spectrum, maintaining a nearly flat and dominant trajectory compared to baselines. Its ability to preserve distal signals—while models like \textit{SASRec} and \textit{LONGER} struggle to close the gap—confirms that the entropy-driven boundary detector successfully identifies and retains informative behavioral shifts regardless of their sequence position. Collectively, these results highlight that while baseline models are often constrained by a dataset's inherent temporal decay, \textbf{DP-Rec}’s adaptive patching provides a robust efficiency advantage across both high-recency and signal-persistent environments.

Long-sequence recommendation is particularly relevant in sparse and long-tail settings, where a user’s older interactions may still contain useful preference signal. We therefore analyze performance stratified by item popularity to assess whether DP-Rec is especially beneficial for long-tail recommendation. Because adaptive patching preserves informative intent transitions rather than truncating the sequence uniformly, DP-Rec demonstrates gains across all the item popularity cohorts(refer Figure \ref{fig:item_popularity}).

\section{Experimental Configuration}                                                                                                                                                                               

Table~\ref{tab:appendix_hyperparameters_transposed} provides a comprehensive overview of the experimental configurations and hyperparameter search spaces across all models.                                       
                                                                                     
\begin{table*}[t]                      
\footnotesize
\centering
\caption{Summary of experimental configurations and hyperparameter search spaces. For direct efficiency comparisons, \textbf{Iso-capacity} variants utilize a $1/1/1$ block structure ($d=64$) to parameter-match
SASRec ($d=64, L=3$).}
\label{tab:appendix_hyperparameters_transposed}
\small
\setlength{\tabcolsep}{6pt}
\begin{tabular}{l ccc cc}
\toprule
\textbf{Configuration} & \multicolumn{3}{c}{\textbf{Fixed-Horizon Models}} & \multicolumn{2}{c}{\textbf{Compression-Based Models}} \\
\cmidrule(lr){2-4} \cmidrule(lr){5-6}
& SASRec \cite{sasrectorch} & HSTU \cite{rectools} & GRU4Rec \cite{rectools} & LONGER & \textbf{DP-Rec} \\
\midrule
Blocks ($L$) & 3 & 3 & $\{1, 3\}$ & $L_{\mathrm{inner}}, L_{\mathrm{outer}} \in \{1, 3\}$ & $L_{\mathrm{enc}}=1, L_{\mathrm{dec}}=1, L_{\mathrm{lat}} \in \{1, 3\}$ \\
Hidden size ($d$) & $\{32, 64\}$ & 64 & 64 & $\{16, 32, 64\}$ & $\{16, 32, 64\}$ \\
Seq. length ($N$) & $\{25, \dots, 500\}$ & Matched & $\{50, 100\}$ & $\{100, \dots, 500\}$ & $\{100, \dots, 500\}$ \\
Patch Budget ($\hat{M}$) & -- & -- & -- & $\{2, \dots, 250\}$ & $\{2, \dots, 250\}$ \\
Pool Initialization & True & True & True & True & True \\
Attention Heads & 1 & 1 & -- & 1 & 1 \\
\midrule
\rowcolor{gray!10} \multicolumn{6}{l}{\textit{Architecture Specifics}} \\
Iso-capacity $L$ & -- & -- & -- & $L_{\mathrm{inner}} = L_{\mathrm{outer}} = 1$ & $L_{\mathrm{enc}} = L_{\mathrm{dec}} = L_{\mathrm{lat}} = 1$ \\
Bridge Mechanism & -- & -- & -- & Cross-Attn & Cross-Attn / Gather-Add \\
Boundary Detector & -- & -- & -- & -- & $d_{\mathrm{bbd}}=8, L_{\mathrm{bbd}}=1, w_e=32$ \\
\midrule
\rowcolor{gray!10} \multicolumn{6}{l}{\textit{Training Specifics (Common to all models)}} \\
Optimizer & \multicolumn{5}{c}{Adam (Learning Rate = $10^{-3}$, Weight Decay = 0)} \\
Batch Size & \multicolumn{5}{c}{128} \\
Max Epochs & \multicolumn{5}{c}{500} \\
Validation Interval & \multicolumn{5}{c}{10 epochs} \\
Early Stopping & \multicolumn{5}{c}{No} \\
Gradient Clipping & \multicolumn{5}{c}{No} \\
\bottomrule
\end{tabular}
\end{table*}

\section{Detailed FLOPs Accounting}
Inference efficiency is measured by FLOPs per forward pass, following
\cite{hoffmann2022training} but excluding non-arithmetic costs (normalization,
residuals, and lookups). For a Transformer with sequence length $n$, average
attention context $m$, hidden dimension $h$, feed-forward expansion
$d_{\mathrm{ff}}$, and $\ell$ layers:

\begin{equation}
\mathrm{FLOPs}(n,\,m)
= n\ell \!\left(
  4h^{2}d_{\mathrm{ff}}
+ 8h^{2}
+ 2h(m+1)
\right).
\label{eq:base_flops}
\end{equation}

\noindent
The three terms correspond to the FFN ($4h^{2}d_{\mathrm{ff}}$), the $Q/K/V/O$ projections ($8h^{2}$), and attention score computation plus value aggregation ($2h(m+1)$). For full causal self-attention, $m = n$ (each of the $n$ positions attends on average to $(n+1)/2$ tokens, giving $4h \cdot \frac{n+1}{2} \equiv 2h(m+1)$ with $m=n$). For sliding-window
attention with window $w$, the bounded context gives $m = w$. We set $d_{\mathrm{ff}}=4$ throughout.

\paragraph{SASRec.}
Full causal attention over $n$ tokens ($m = n$):
\begin{equation}
\mathrm{FLOPs}_{\mathrm{SAS}}
= n\ell \!\left( 4h^{2}d_{\mathrm{ff}} + 8h^{2} + 2h(n+1) \right).
\label{eq:sas_flops}
\end{equation}

\paragraph{HSTU.}
Same causal backbone with a $1.1\times$ overhead for gating:
\begin{equation}
\mathrm{FLOPs}_{\mathrm{HSTU}} = 1.1 \times \mathrm{FLOPs}_{\mathrm{SAS}}.
\label{eq:hstu_flops}
\end{equation}

\paragraph{LONGER.}
Let $k$ be the group size, $G = \lceil n/k \rceil$ the number of groups, $h_{\ell}, \ell_{\ell}$ the local (inner-transformer) hidden size and layers, and $h_g, \ell_g$ the global hidden size and layers. Inner-group encoding uses full causal attention over $k$ tokens ($m=k$); global reasoning uses full causal attention over $G$ groups ($m=G$):
\begin{equation}
\mathrm{FLOPs}_{\mathrm{LNG}}
=
G \cdot f_{\ell_{\ell},h_{\ell}}(k,\,k)
\;+\; 2nGh_g
\;+\; f_{\ell_g,h_g}(G,\,G),
\label{eq:longer_flops}
\end{equation}
where $f_{\ell,h}(n,m) = \ell(4h^{2}d_{\mathrm{ff}} + 8h^{2} + 2h(m+1))$ is the per-token cost from Eq.~\ref{eq:base_flops}, and $2nGh_g$ is the token-to-group cross-attention cost.

\paragraph{DP-Rec.}
Let $n_p$ be the average patch size, $M = \lceil n/n_p \rceil$ the number of latent patches,
$d_{\mathrm{bbd}}, L_{\mathrm{bbd}}, w_e$ the boundary detector hidden size, layers, and window,
$d, L_{\mathrm{lat}}$ the latent transformer hidden size and layers,
$d_{\mathrm{loc}}, L_{\mathrm{loc}}, w$ the local encoder/decoder hidden size, layer count, and sliding window.
The total per-sequence cost decomposes as:
\begin{equation}
\mathrm{FLOPs}_{\mathrm{DP}}
= n\!\left(
  f_{\mathrm{bbd}}
+ \frac{f_{\mathrm{lat}}}{n_p}
+ f_{\mathrm{enc}}
+ f_{\mathrm{dec}}
+ f_{\mathrm{brg}}
\right),
\label{eq:dp_flops}
\end{equation}
where each per-token term is an instance of Eq.~\ref{eq:base_flops}:
\begin{align}
f_{\mathrm{bbd}}
&= L_{\mathrm{bbd}}\!\left(
     4d_{\mathrm{bbd}}^{2}d_{\mathrm{ff}} + 8d_{\mathrm{bbd}}^{2} + 2d_{\mathrm{bbd}}(w_e+1)
   \right),
\label{eq:f_bbd}\\[3pt]
f_{\mathrm{lat}}
&= L_{\mathrm{lat}}\!\left(
     4d^{2}d_{\mathrm{ff}} + 8d^{2} + 2d(M+1)
   \right),
\label{eq:f_lat}\\[3pt]
f_{\mathrm{enc}}
&= L_{\mathrm{loc}}\!\left(
     4d_{\mathrm{loc}}^{2}d_{\mathrm{ff}} + 8d_{\mathrm{loc}}^{2} + 2d_{\mathrm{loc}}(w+1)
   \right),
\label{eq:f_enc}\\[3pt]
f_{\mathrm{dec}}
&= L_{\mathrm{loc}}\!\left(
     4d_{\mathrm{loc}}^{2}d_{\mathrm{ff}} + 8d_{\mathrm{loc}}^{2} + 2d_{\mathrm{loc}}(w+1)
   \right).
\label{eq:f_dec}
\end{align}
The global cost $f_{\mathrm{lat}}$ is amortised over $n_p$ tokens per patch (the $1/n_p$ factor
in Eq.~\ref{eq:dp_flops}). The bridge term $f_{\mathrm{brg}}=0$ under the default gather-add
bridge. Although the behavioral boundary detector is frozen during training, its cost is
included in the total inference budget.

\begin{table}[t]
\centering
\caption{%
Concrete FLOPs at $n=200$, $d_{\mathrm{ff}}=4$.
SASRec/HSTU: $d=64$, $\ell=3$.
LONGER: $d=d_{\mathrm{loc}}=64$, $L_{\mathrm{lat}}=L_{\mathrm{loc}}=1$, $k=4$, $G=50$.
DP-Rec: $d=d_{\mathrm{loc}}=64$, $L_{\mathrm{lat}}=L_{\mathrm{loc}}=1$, $w=32$,
$d_{\mathrm{bbd}}=8$, $L_{\mathrm{bbd}}=1$, $w_e=32$, $M=10$.
}
\label{tab:flops}
\begin{tabular}{lccc}
\toprule
\textbf{Model} & \textbf{FLOPs} & \textbf{vs.\ SASRec} & \textbf{Complexity} \\
\midrule
SASRec  & $74{,}419{,}200$ & $1.00\times$ & $O(n^2 h \ell)$ \\
HSTU    & $81{,}861{,}120$ & $1.10\times$ & $O(n^2 h \ell)$ \\
LONGER  & $26{,}310{,}400$ & $0.35\times$ & $O(nk d_{\mathrm{loc}} L_{\mathrm{loc}} + G^2 d L_{\mathrm{lat}})$ \\
DP-Rec  & $42{,}421{,}120$ & $0.57\times$ & $O(nw\, d_{\mathrm{loc}} \cdot 2L_{\mathrm{loc}} + M^2 d\, L_{\mathrm{lat}})$ \\
\bottomrule
\end{tabular}
\end{table}

\end{document}